\documentclass[letterpaper,11pt]{article}

\usepackage{amsfonts,amsmath,amssymb,amsthm,graphicx}
\usepackage{mathtools}
\usepackage{mathrsfs}
\usepackage{ifthen}
\usepackage{changepage}
\usepackage{enumerate}
\usepackage{scalerel}
\usepackage{accents}
\usepackage[margin=1in]{geometry}
\usepackage{enumitem}
\usepackage{natbib}
\usepackage{subcaption} 
\usepackage{float}      
\usepackage[utf8]{inputenc} 
\usepackage[T1]{fontenc} 
\usepackage{csquotes}
\usepackage{comment}
\usepackage{bbm}
\usepackage{lmodern}
\usepackage{array}
\usepackage{makecell}
\usepackage{csquotes}
\usepackage{booktabs}
\usepackage{xfrac}
\usepackage{multirow}
\usepackage[colorlinks=true, linkcolor= blue!70!black, citecolor=blue!70!black,urlcolor=blue,breaklinks=true]{hyperref}

\usepackage[suppress]{color-edits}
\addauthor{cp}{cyan}
\addauthor{wt}{purple}

\usepackage{pgfplots}
\usepackage{tikz}
\pgfplotsset{compat=1.17}
\usetikzlibrary{patterns}
\usepgfplotslibrary{fillbetween}
\usetikzlibrary{intersections}
\usetikzlibrary{arrows.meta,positioning}
\usetikzlibrary{arrows, decorations.markings}

\definecolor{highlightred}{RGB}{180, 0, 0}
\usepackage{tcolorbox}

\usepackage{thm-restate}
\usepackage{cleveref}
\crefname{claim}{claim}{claims}

\tikzstyle{vecArrow} = [thick, decoration={markings,mark=at position
   1 with {\arrow[semithick]{open triangle 60}}},
   double distance=1.4pt, shorten >= 5.5pt,
   preaction = {decorate},
   postaction = {draw,line width=1.4pt, white,shorten >= 4.5pt}]
\tikzstyle{innerWhite} = [semithick, white,line width=1.4pt, shorten >= 4.5pt]

\allowdisplaybreaks

\theoremstyle{plain}
\newtheorem{theorem}{Theorem}[section]          
              
\newtheorem{proposition}[theorem]{Proposition}

\theoremstyle{definition} 
\newtheorem{definition}[theorem]{Definition}

\theoremstyle{plain}

\newcommand{\squishlist}{
\begin{list}{{$\bullet$}}
{\setlength{\itemsep}{3pt}      \setlength{\parsep}{1pt}
\setlength{\topsep}{1pt}       \setlength{\partopsep}{0pt}
\setlength{\leftmargin}{1em} \setlength{\labelwidth}{1em}
\setlength{\labelsep}{0.5em} } }
\newcommand{\squishend}{  \end{list}}

\newcommand{\xhdr}[1]{\paragraph{\bf #1}}

\newcommand{\Payoff}[2][]{\text{\bf Payoff}\ifthenelse{\not\equal{}{#1}}{_{#1}}{}\!\left[{\def\givenn{\middle|}#2}\right]}
\newcommand{\Reg}[2][]{\text{\bf Reg}\ifthenelse{\not\equal{}{#1}}{_{#1}}{}\!\left[{\def\givenn{\middle|}#2}\right]}

\newcommand{\state}{\theta}
\newcommand{\statespace}{\Theta}
\newcommand{\prior}{\lambda}
\newcommand{\outcome}{\omega}
\newcommand{\outcomespace}{\Omega}

\newcommand{\signal}{s}
\newcommand{\signalspace}{S}
\newcommand{\inforstructure}{\pi}
\newcommand{\inforstructureClass}{\Pi}
\newcommand{\inforstructureClassCI}{\Pi_{\textsc{CI}}}
\newcommand{\inforstructureClassBO}{\Pi_{\textsc{BO}}}
\newcommand{\report}{x}
\newcommand{\priormean}{\mu}
\newcommand{\Bernoulli}{\textsc{Bern}}
\newcommand{\aggregator}{f}
\newcommand{\bayesaggregator}{\aggregator^{\textsc{Bayes}}}
\newcommand{\bayesaggregatornew}{\aggregator^{\primed, \textsc{Bayes}}}
\newcommand{\loss}[2][]{\ell\ifthenelse{\not\equal{}{#1}}{_{#1}}{}\!\left({\def\givenn{\middle|}#2}\right)}
\newcommand{\avgloss}[2][]{L\ifthenelse{\not\equal{}{#1}}{_{#1}}{}\!\left({\def\givenn{\middle|}#2}\right)}
\newcommand{\expertoneweight}{\alpha}
\newcommand{\experttwoweight}{\beta}
\newcommand{\baserate}{\gamma}
\newcommand{\reals}{\mathbb{R}}

\newcommand{\logit}{\textsc{logit}}
\newcommand{\sigmoid}{\sigma}

\newcommand{\instance}{I}

\newcommand{\prob}[2][]{\text{Pr}\ifthenelse{\not\equal{}{#1}}{_{#1}}{}\!\left[{\def\givenn{\middle|}#2}\right]}
\newcommand{\expect}[2][]{\mathbb{E}\ifthenelse{\not\equal{}{#1}}{_{#1}}{}\!\left[{\def\givenn{\middle|}#2}\right]}

\newcommand{\tparen}{\big}
\newcommand{\tprob}[2][]{\text{Pr}\ifthenelse{\not\equal{}{#1}}{_{#1}}{}\tparen[{\def\given{\tparen|}#2}\tparen]}
\newcommand{\texpect}[2][]{\mathbb{E}\ifthenelse{\not\equal{}{#1}}{_{#1}}{}\tparen[{\def\given{\tparen|}#2}\tparen]}

\newcommand{\sprob}[2][]{\text{Pr}\ifthenelse{\not\equal{}{#1}}{_{#1}}{}[#2]}
\newcommand{\sexpect}[2][]{\mathbb{E}\ifthenelse{\not\equal{}{#1}}{_{#1}}{}[#2]}

\newcommand{\dd}{{\mathrm d}}

\newcommand{\cc}[1]{\ensuremath{\mathsf{#1}}} 

\newcommand{\primed}{^{\dagger}}
\newcommand{\doubleprimed}{^{\ddagger}}
\newcommand{\supp}{\cc{supp}}

\begin{document}

\title{Prior-Agnostic Robust Forecast Aggregation}
\author{
Zhi Chen\thanks{Chinese University of Hong Kong. 
Email: {\tt zhi.chen@cuhk.edu.hk}}
\and
Cheng Peng\thanks{Chinese University of Hong Kong. Email: {\tt chengpeng@link.cuhk.edu.hk}}
\and
Wei Tang\thanks{Chinese University of Hong Kong. Email: {\tt weitang@cuhk.edu.hk}}}
\date{}


\maketitle

\begin{abstract}

Robust forecast aggregation combines the predictions of multiple information sources to perform well in the worst case across all possible information structures. 
Previous work largely focuses on settings with a known binary state space, where the state is either $0$ or $1$. We study prior-agnostic robust forecast aggregation in which the aggregator observes only experts' reports, yet is ignorant of both the underlying joint information structure and the full prior, including the underlying state space.
Unlike the standard model that fixes the binary state space $\{0, 1\}$, we allow the (binary) unknown state values to be arbitrary numbers in $[0, 1]$, so the same reported probability may correspond to very different realized outcome frequencies across environments.

Our main contribution is a simple, explicit, closed-form log-odds aggregator that linearly pools forecasts in logit space, together with (nearly-)tight minimax-regret guarantees across three knowledge regimes. 
We first show that under conditionally independent (CI) signals, robust aggregation with an unknown state space is strictly harder than in the known-state setting by establishing a larger lower bound, and our aggregation rule can achieve worst-case regret of $0.0255$. Along the way, we also characterize tight regret bounds for Blackwell-ordered structures and for general information structures.
In the classical setting with known state space $\{0, 1\}$, our aggregator achieves regret strictly below $0.0226$ for CI structures.
To the best of our knowledge, this is the first explicit closed-form aggregator that achieves a regret upper bound strictly less than $0.0226$.
Finally, we extend the model to allow the aggregator to additionally know each expert's marginal forecast distribution; in this setting, with the CI structures, we show that a generalized log-odds rule achieves regret of $0.0228$, and we complement this guarantee with a lower bound of $0.0225$.

\end{abstract}

\newpage

\section{Introduction}
\label{sec:intro}
When checking tomorrow's weather, reputable sources often sharply disagree on the likelihood of rain, leaving the downstream decision-maker unsure how to combine these multiple predictions into a single, reliable forecast.
In practice, each expert (e.g., a forecasting service) forms their forecast based on observable covariates describing the current atmospheric state, such as pressure patterns, humidity levels, and radar signatures, which collectively encode a latent ``likelihood of rain.''
Importantly, these informational signals are generated from such features before the rain outcome is realized: they are informative about rain only through the underlying feature state,\footnote{For example, the weather forecasting service may observe a summary of atmospheric measurements, such as upstream radar reflectivity and local humidity profiles, which together provide a noisy cue about tomorrow's rainfall.}
rather than being conditionally drawn from the eventual binary rain/no-rain outcome itself.

The aggregator, however, typically observes only the experts' reported probabilities, not the underlying features or the mapping from features to signals, while still hoping to choose an aggregation rule that performs well uniformly across all possible (and unknown) signal structures that could have produced the experts' reports. 
If the features perfectly revealed whether it would rain, then the latent uncertainty collapses to a binary state space $\{0, 1\}$, recovering the classic robust aggregation problem with a known state space.
This problem was first studied in \citet{ABS-18} and later extended by \citet{NR-22,KWW-24,GHHKSY-25}. In reality, the features only partially predict rain, and because the aggregator does not observe them, they face a genuinely prior-agnostic robust aggregation problem in which the relevant ``state'' is the (unknown) latent likelihood encoded by those features, rather than a fixed and known binary variable. 

This same feature-outcome separation appears in many real-world applications. 
For example, election forecasters base their predictions on poll responses, turnout proxies, demographic composition, and economic indicators, not on the election outcome that will only be observed later; physicians form prognoses from symptoms, lab tests, imaging, and patient history, not from the future health outcome itself; and financial analysts recommend portfolios using earnings reports, macro signals, order flow, and risk exposures, not the realized returns. In each case, experts' forecasts are best viewed as posterior beliefs induced by intermediate features, while the aggregator must reconcile conflicting posteriors without observing either those features or the experts' information-generating processes.

Motivated by the above observations, we generalize the robust forecast aggregation problem studied in \citet{ABS-18}, which assumes a known and fixed binary state space $\{0, 1\}$, to a setting in which the state space is unknown to the aggregator.
More formally, we consider a stochastic environment in which there is an unobserved state that captures the underlying ``situation'' of the world relevant for predicting a binary outcome. 
Here, the state should be thought of as a latent driver of the outcome's likelihood, such as an underlying atmospheric configuration that makes rain more or less likely, and, more generally, as any hidden condition that determines the event's likelihood. 
To capture such latent uncertainty while keeping the model parsimonious, we allow the latent likelihood to take one of two (a priori unknown) states, drawn from an unknown prior distribution.
The outcome is then generated conditionally on the realized state via a Bernoulli distribution, with a success probability equal to the likelihood encoded by that state.
There are two experts, each observing a private signal stochastically related to the state through an unknown joint information structure, and then reporting a probabilistic forecast equal to their posterior belief about the outcome. 
The aggregator observes only the experts' reported probabilities, and does not observe the state, the experts' raw signals, the prior, or the signal generation process, and must map the pair of reports into a single forecast. 
We evaluate the aggregator's performance by squared-error loss, and regret is defined relative to an omniscient Bayesian benchmark that knows the full environment and has access to the experts' signals. 
Our goal is to design an aggregation function that achieves small worst-case regret uniformly over all priors (including the unknown states) and information structures (or a specified class of structures).

\subsection{Our Contributions}
We initiate the study of prior-agnostic robust forecast aggregation with an unknown state space, 
motivated by settings in which experts' reports are induced by latent feature states that are not directly observed by the aggregator. 
Relative to the classic robust aggregation model with a known binary state space, our formulation treats the ``state'' as an unknown latent likelihood that corresponds to the event/outcome probability.
Our goal is to design an aggregation rule whose performance is uniformly good over all priors and information structures consistent with the experts' reports, including uncertainty about the underlying state space itself.

Our first technical contribution is a simple, explicit, closed-form log-odds aggregation function that linearly pools experts' forecasts in logit space. 
In the setting of conditionally independent signals (i.e.,  experts' signals are distributed independently conditional on the realized state), 
we provide a near-tight minimax regret bound: our aggregation function achieves worst-case regret $0.0255$ (\Cref{thm:regret ub unknown}), 
and we complement this upper bound by showing that the optimal scheme cannot do much better.
In particular, we show a lower bound that no aggregation scheme can achieve regret below $\sfrac{31}{1326}\approx 0.0234$ (\Cref{thm:regret lb unknown}).
This lower bound demonstrates a strict separation between the unknown-state problem and the known-state benchmark:
under the setting with a known binary state space $\{0, 1\}$,
\citet{ABS-18} have established a regret lower bound of $\frac{1}{8}(5\sqrt{5}-11)\approx 0.0225$, and \citet{GHHKSY-25} present an algorithmic approach that can achieve regret $0.0226$, which is strictly below $0.0234$.
Taken together, these results show that robust aggregation is strictly harder when the state space is not known a priori in environments with conditionally independent signals.

We also evaluate how aggregation schemes designed for the known-state setting perform when the state space is unknown.
As shown in \Cref{tab:regret_comparison}, all methods we evaluate, including those proposed in \citet{ABS-18, KWW-24}, incur worst-case regret that is strictly above $0.03$, which is substantially larger than the regret achieved by our proposed aggregation scheme.
We also apply our log-odds aggregation scheme back to the classical known-binary-state setting.
In the known state space setting $\{0,1\}$, the same log-odds family (with an appropriate parameter choice) achieves regret that is strictly below $0.0226$ (\Cref{prop:known_binary_upper_bound}).
To the best of our knowledge, this is the first regret upper bound in this regime achieved by an explicit closed-form aggregator.
\footnote{\cpedit{
Consistent with the approach adopted in prior works~\citep{ABS-18, GHHKSY-25}, we numerically estimate and compare the worst-case regrets of our proposed aggregator against other common benchmarks.
In all of our numerical experiments, we evaluate the worst-case performance of our proposed aggregation function using a two-stage, high-precision procedure designed to avoid missing the true worst-case instance. 
First, we run a global heuristic search over the relevant family of information structures to locate a small set of candidate worst-case instances. 
Second, we independently validate each candidate via a fine-grained grid refinement (discretization step size of $10^{-5}$).
For example, in the known $\{0, 1\}$ state-space setting, this pipeline consistently identifies a worst-case regret of $0.022599$, and the subsequent refinement confirms that this value is numerically stable and strictly below 0.0226.
}}

Beyond conditional independence, we also obtain sharp minimax characterizations in the unknown-state-space setting for two additional settings with different classes of information structures. 
For Blackwell-ordered information structures (one expert is uniformly more informative than the other), we show that not knowing the state values does not hurt the worst-case aggregation performance: the minimax regret is exactly $\frac{1}{8}(5\sqrt{5}-11)$, matching the known-$\{0,1\}$ setting (\Cref{pro:blackwell_lowerupperbound_unknown}). 
Moreover, this bound is achieved by directly applying the precision-weighted aggregation scheme proposed in \citet{ABS-18}, as the scheme depends solely on the observed forecasts (via their ``precision,'' i.e., inverse variance) and not on the unknown state values or prior. 
For fully general information structures (no restrictions on correlation), the robust aggregation is fundamentally limited: the minimax regret is exactly $\sfrac{1}{4}$. This is achieved by the trivial constant rule that always outputs $\sfrac{1}{2}$, and the matching lower bound follows from an inclusion argument (the adversary can restrict attention to the classic binary-state hard instance) (\Cref{pro:general_lbub_unknow}), where a lower bound $\sfrac{1}{2}$ in the known state space setting is known \citep{ABS-18}.

Finally, we explore the setting in which the aggregator additionally knows each expert's marginal forecast distribution (i.e., the unconditional distribution of the expert's posterior mean).
Note that this extra information can indeed make some environments essentially trivial. 
In Blackwell-ordered settings, the aggregator can compare the informativeness implied by the two marginal distributions (via the induced Blackwell order), identify the better-informed expert, and simply follow that expert, achieving zero regret. 
Likewise, in the classical conditionally independent model with a known binary state space $\{0, 1\}$, observing an expert's marginal forecast distribution pins down the prior mean via Bayes plausibility. Since the omniscient Bayesian benchmark under conditional independence depends on the information structure only through this prior mean, aggregation again becomes trivial once the marginals are known.

Somewhat surprisingly, however, knowing the marginals does not, in general, eliminate worst-case difficulty in our unknown-state setting. 
For general information structures, the minimax regret remains exactly $0.25$ even with known marginals.
Even under conditionally independent structures, we prove a minimax lower bound of $\frac{1}{8}(5\sqrt{5}-11)$ (\Cref{thm:regret lb unknown w marg}): unlike the known-$\{0, 1\}$ setting, the adversary can exploit different state spaces that induce the same marginal forecast distributions, so the marginals do not reveal which environment is realized.
We then show that a generalized log-odds aggregation function, which in fact uses only the inferred prior mean (rather than the full marginals), achieves a nearly matching guarantee: for a suitable choice of parameters, it attains the worst-case regret of at most $0.022763$ (\Cref{pro:con_ub_knowmarg}).

We summarize most of our results in the following \Cref{tab:summary_regret}.


\begin{table}[h!]
\centering
\small
\renewcommand{\arraystretch}{1.5}
\setlength{\tabcolsep}{3pt}

\begin{tabular}{@{}lcccccc@{}}
\toprule
\multirow{2}{*}{\textbf{Setting}} & 
\multicolumn{2}{c}{\textbf{\begin{tabular}[c]{@{}c@{}}Conditionally\\ Independent\end{tabular}}} & 
\multicolumn{2}{c}{\textbf{\begin{tabular}[c]{@{}c@{}}Blackwell-Ordered\end{tabular}}} & 
\multicolumn{2}{c}{\textbf{\begin{tabular}[c]{@{}c@{}}General \end{tabular}}} \\ 
\cmidrule(lr){2-3} \cmidrule(lr){4-5} \cmidrule(lr){6-7}

 & \textbf{LB} & \textbf{UB} & 
 \multicolumn{2}{c}{\textbf{Tight Bound}} & 
 \multicolumn{2}{c}{\textbf{Tight Bound}} \\ 
\midrule

\begin{tabular}[c]{@{}l@{}}Known state space\\ $\statespace=\{0, 1\}$\end{tabular} & 
$\dfrac{5\sqrt{5}-11}{8}\primed$ & 
\begin{tabular}[c]{@{}c@{}}
    $0.022599$ \\ 
    \scriptsize{(\Cref{prop:known_binary_upper_bound})}
\end{tabular} & 
\multicolumn{2}{c}{$\dfrac{5\sqrt{5}-11}{8}$} & 
\multicolumn{2}{c}{$0.25$} \\ 
\midrule

\begin{tabular}[c]{@{}l@{}}Unknown  state space\end{tabular} & 
\begin{tabular}[c]{@{}c@{}}
    $\dfrac{31}{1326}\doubleprimed$ \\ 
    \scriptsize{(\Cref{thm:regret lb unknown})}
\end{tabular} & 
\begin{tabular}[c]{@{}c@{}}
    $0.025512$ \\ 
    \scriptsize{(\Cref{thm:regret ub unknown})}
\end{tabular} & 
\multicolumn{2}{c}{\begin{tabular}[c]{@{}c@{}}
    $\dfrac{5\sqrt{5}-11}{8}$ \\ 
    \scriptsize{(\Cref{pro:blackwell_lowerupperbound_unknown})}
\end{tabular}} & 
\multicolumn{2}{c}{\begin{tabular}[c]{@{}c@{}}
    $0.25$ \\ 
    \scriptsize{(\Cref{pro:general_lbub_unknow})}
\end{tabular}} \\ 
\midrule

\begin{tabular}[c]{@{}l@{}}Unknown state space\\ Known marg. predic. dist.\end{tabular} & 
\begin{tabular}[c]{@{}c@{}}
    $\dfrac{5\sqrt{5}-11}{8}$ \\ 
    \scriptsize{(\Cref{thm:regret lb unknown w marg})}
\end{tabular} & 
\begin{tabular}[c]{@{}c@{}}
    $0.022763$ \\ 
    \scriptsize{(\Cref{pro:con_ub_knowmarg})}
\end{tabular} & 
\multicolumn{2}{c}{$0$} & 
\multicolumn{2}{c}{\begin{tabular}[c]{@{}c@{}}
    $0.25$ \\ 
    \scriptsize{(\Cref{pro:general_lbub_knowmarg})}
\end{tabular}}  \\ 
\bottomrule
\multicolumn{7}{l}{\footnotesize $\primed \approx 0.022542$; \quad $\doubleprimed \approx 0.023379$.}
\end{tabular}
\caption{Summary of regret bounds. Baseline results are rounded to 4 decimal places; our results are reported to 6 decimal places. All rounding follows the standard nearest value rule.}
\label{tab:summary_regret}
\end{table}

\subsection{Related Work}

Our work closely relates to the literature on robust forecast aggregation, which aims to find aggregators that perform well in the worst case across all possible information structures. 
The foundational question of designing such an aggregator was first studied in \citet{ABS-18}.
Focusing on the two-expert setting with a binary state space $\{0, 1\}$ and conditionally independent signals, they propose low-regret aggregation rules and establish corresponding upper and lower bounds on the additive regret. 
Specifically, they show that an average prior scheme achieves a worst-case regret of $0.0260$, while proving a lower bound of 
$\frac{5\sqrt{5}-11}{8}\approx 0.0225$.
Building on this setting, 
recent work by \citet{GHHKSY-25} develop an efficient algorithmic framework that can achieve a regret of $0.0226$.
These two works are the closest to ours, since we also focus on conditionally independent signals. Our contribution differs in that we study an unknown state-space setting and show that previously proposed aggregation schemes can perform poorly in this regime.

Complementing these approaches, several studies consider settings in which the aggregator is endowed with additional information.
\citet{LR-22} study robust prediction when the aggregator understands each information source in isolation but is uncertain about the correlation across sources.
 
\citet{DIL-21} study the robust decision making under uncertain correlations among information sources where the uncertainty lies solely in the correlation structure among information sources.
\citet{LR-21} propose a maximum-likelihood method for combining multiple forecasts (given a prior) that rationalizes the observed forecasts via the most likely underlying information structure.

Beyond additive regret, other works explore alternative robustness criteria. 
For example, \citet{NR-22,FMNW-25} study a ratio-based robustness notion for a class of information structures they term projective substitutes.
\citet{ABTZ-25} study binary decision aggregation problem under different standard robustness notions, including maximin, regret minimization, and approximation ratio, and establish that the simple random dictator rule (which selects one agent's recommendation uniformly at random) is the unique optimal mechanism when the number of agents is large.

More recently, the literature has broadened its scope to account for behavioral or strategic deviations by experts. 
For example, \citet{KWW-24} study aggregation where experts exhibit base rate neglect.
\citet{GK-25} consider settings with adversarial experts and characterize the robustly optimal aggregation mechanisms under various loss functions.
Complementary to this strand, another line of research leverages second-order information, e.g., agents' beliefs about others’ answers, to improve aggregation performance.
This line of work is initiated with Bayesian Truth Serum \citep{P-04} and the ``surprisingly popular'' rule \citep{PSM-17}, and has more recently developed a broader toolkit that uses second- and higher-order belief elicitation to aggregate opinions in finite populations, including settings with heterogeneous agent types (see, e.g., \citealp{PS-19,WLC-21,WMH-22,PCK-24,APSTX-25}).


In addition to aggregating forecasts in a robust manner, forecast aggregation has traditionally been also studied through two complementary lenses: axiomatic and Bayesian. The axiomatic literature posits normative desiderata for aggregation rules, e.g., unanimity preservation and variants of independence, 
and characterizes the (often essentially unique) rules that satisfy them, with linear pooling (averaging) emerging as a canonical solution under standard axioms (see, e.g., \citealp{AW-80,G-84,DL-16}). 
In contrast, the Bayesian literature models how forecasts are generated from underlying signals and prescribes aggregation by Bayesian updating; most results in this tradition are parametric, optimizing aggregation within a specified family of information structures (e.g., common-knowledge correlation patterns or exponential-family sampling models), which can limit robustness when the true dependence structure is misspecified (see, e.g., \citealp{SBFMTU-14,FCK-15,EPSU-16,SPU-16}).


\section{Preliminary}
\label{sec:prelim}

We consider a stochastic environment where the world randomly generates a \emph{state} $\state$ from a finite state space $\statespace$.
In this work, we focus on a binary state space, so that $|\statespace| = 2$.
Once a state is realized, it further induces a random \emph{binary outcome}, which is unobservable to all players. 
Specifically, for each state $i \in [2]$, the binary outcome $\outcome \in \outcomespace \triangleq \{0, 1\}$ is drawn from a Bernoulli distribution with mean $\state_i \in [0, 1]$.  
Without loss of generality, we assume that the states are sorted in non-decreasing order of their Bernoulli means: $\state_1 \le \state_2$.
Since each state can be identified by its Bernoulli mean, 
we henceforth refer to the states directly as  $\state_1$ and $\state_2$, and write the state space as $\statespace = \{\state_1, \state_2\}$.
We denote by $\prior_i \in [0, 1]$ the prior probability of the state $\state_i$ for $i\in[2]$ \cpedit{and let $\mu$ be the prior mean of the state, i.e., $\priormean \triangleq \expect[\state\sim\prior]{\state}$}. 

\xhdr{Information structure}
There are two experts, each receiving a private signal that provides information about the current world state. 
We use an {\em information structure} to characterize the relationship between the experts' {\em private} signals and the underlying states. 
Formally, an information structure for two experts is a tuple $(\signalspace, \inforstructure)$, where $\signalspace = \signalspace_1 \times \signalspace_2$ is a $2$-dimensional signal space and $\inforstructure\in\Delta(\statespace\times \signalspace)$ is a joint distribution over states and signal profiles. 
Each expert $i\in[2]$ receives a private signal $\signal_i\in\signalspace_i$ drawn according to $\inforstructure$.
Upon receiving their private signal $\signal_i$ and knowing the information structure $\inforstructure$, each expert $i$ provides his forecast $\report_i(\signal_i \mid \inforstructure)$ about the underlying outcome $\outcome$. 
These forecasts are their Bayesian posterior for outcome: since the outcome space is binary, each expert's forecast is simply their conditional probability that $\outcome = 1$, i.e., 
\begin{align*}
    \report_i(\signal_i \mid \inforstructure) \triangleq 
    \prob[\inforstructure]{\outcome = 1 \mid  \signal_i}~.
\end{align*}

\xhdr{Aggregation and loss}
An aggregator, who does not observe the signals or the true state, must form her own forecast about the underlying outcome based solely on the experts' reported forecasts $\report(\signal) = (\report_1(\signal_1), \report_2(\signal_2))$.
Unlike in prior work, the aggregator is both information-structure-agnostic and prior-agnostic: she knows only that the state space is binary but does not know the specific values of $\state_1$ and $\state_2$, nor the prior probabilities $\prior_1, \prior_2$, nor the information structure. 
More formally, the aggregator uses a deterministic aggregation function $\aggregator: [0, 1]^2 \rightarrow [0, 1]$ that is a mapping from the two experts' forecasts $(\report_1, \report_2)\in[0, 1]^2$ to her own forecast.

We evaluate the aggregator's performance using a loss function $\loss{\cdot, \cdot}: [0, 1] \times \outcomespace\rightarrow\reals_+$ where $\loss{\report, \outcome} \in \reals_+$ quantifies the loss when the induced outcome is $\outcome\in\outcomespace$ and the aggregator's prediction is $\report \in [0, 1]$.
Throughout this work, we adopt the squared loss  $\loss{\report, \outcome} = (\report - \outcome)^2$ as our loss function.

We benchmark the aggregator's performance against the omniscient (Bayesian optimal) aggregator, who knows the full model: the exact state space $\statespace = \{\state_1, \state_2\}$, the prior $\prior$, the information structure $\inforstructure$, and observes all the signals of all the experts.
With the square loss function, the omniscient aggregator's optimal forecast is her Bayesian posterior expectation of $\outcome$ 
given the received reports from the experts and the underlying information structure. 
We denote by $\bayesaggregator: [0, 1]^2 \rightarrow [0, 1]$ such Bayesian optimal aggregation: 
\begin{align*}
    \bayesaggregator(\report(\signal)) 
    \triangleq \prob[\inforstructure]{\outcome = 1\mid \signal}
\end{align*}
where $\signal = (\signal_1, \signal_2)$  is the full signal profile.
Note that while $\bayesaggregator$ is a function of the signals $\signal$, it can be equivalently viewed as a function of the reported forecasts $\report(\signal)$, since the reports are deterministic functions of the signals under a fixed $\inforstructure$.

\xhdr{Prior-agnostic robust forecasting aggregation problem}
With these definitions, we are ready to formulate the aggregator's prior-agnostic robust forecasting problem.

Given a state space $\statespace\subset [0, 1]^2$ and an information structure $\inforstructure\in\Delta(\statespace\times\signalspace)$, 
the expected loss of the aggregation function $\aggregator$ is
\begin{align*}
    \avgloss{\aggregator \mid \statespace, \inforstructure}
    \triangleq 
    \expect[(\state, \signal)\sim \inforstructure, \outcome\sim \Bernoulli(\state)]{
    \loss{\aggregator(\report(\signal)), \outcome} - 
    \loss{\bayesaggregator(\report(\signal)), \outcome}}~.
\end{align*}
Here $\Bernoulli(\state)$ denotes a Bernoulli distribution with the mean $\state$. 

Following the prior work (see, e.g., \citealp{ABS-18,GHHKSY-25}), we allow the aggregator to have partial knowledge of the information structure, modeled as a known class $\inforstructureClass$ of possible structures. 
The prior-agnostic robust forecasting problem solves the following the minimax problem, which aims to find an aggregation function $\aggregator$ that minimizes this worst-case regret over all admissible state spaces and information structures in $\inforstructureClass$:
\begin{align*}
    \min\nolimits_{\aggregator} \; 
    \Reg[\inforstructureClass]{\aggregator}~.
\end{align*}
where $\Reg[\inforstructureClass]{\aggregator}  \triangleq \sup\nolimits_{\statespace\subset [0, 1]^2, \inforstructure\in \inforstructureClass} \; \avgloss{\aggregator \mid \statespace, \inforstructure}$.
When the state space $\statespace$ is known and fixed to the aggregation and set to be $\{0, 1\}$, then our prior-agnostic robust forecasting problem reduces to the robust forecasting aggregation problem studied in \citet{ABS-18,GHHKSY-25}.

\xhdr{An exemplification of our framework}
We illustrate the prior-agnostic robust forecasting problem through an emergency-department triage setting,
where a \emph{human} decision maker aggregates two probabilistic risk assessments.

Consider a human clinician (the aggregator) who must combine two probability forecasts from two sources (the experts), e.g., an AI risk score and a specialist's judgment, to assess whether a patient will experience an adverse event. The latent state \(\state\in\statespace=\{\state_1,\state_2\}\) represents the patient's underlying risk regime (low vs.\ high), and the realized binary outcome \(\outcome\in\outcomespace=\{0,1\}\) indicates whether the adverse event occurs. Each expert \(i\in\{1,2\}\) observes a private signal \(\signal_i\in \signalspace_i\) (such as model-specific features or clinical cues) and reports a forecast \(\report_i\in[0,1]\); the aggregator observes only the pair \((\report_1,\report_2)\) and outputs a final forecast using an aggregation rule.

In practice, the clinician typically does not know the numerical values of the baseline risks \(\state_1,\state_2\) (they vary across hospitals, populations, protocols, and over time), nor the prior likelihood of each regime (e.g., case mix shifts during an outbreak). Moreover, she rarely knows the information structure linking the experts’ private signals to the latent state: the two sources may share data, overlap in cues, or change after retraining or workflow updates, creating unknown dependence and calibration patterns. The prior-agnostic robust forecasting problem therefore asks for an aggregation rule that performs well against an omniscient benchmark uniformly over a plausible class of such information structures.

\xhdr{The expected loss simplification}
With the squared loss, we can simplify the expected loss of
the aggregation function $\avgloss{\aggregator \mid \statespace, \inforstructure}$ as follows:
\begin{restatable}[Regret simplification \citealp{ABS-18,KWW-24,GHHKSY-25}]{lemma}{lemregretequivalence}
\label{lem:regret_equivalence}
Given any state space $\statespace \subset [0, 1]^2$ and any information structure $\inforstructure$, we have that
\begin{align*}
    \avgloss{\aggregator \mid \statespace, \inforstructure}
    &= 
    \expect[(\state, \signal)\sim \inforstructure, \outcome\sim \Bernoulli(\state)]{\left(\aggregator(\report(\signal)) - \bayesaggregator(\report(\signal))\right)^2}~.
\end{align*}
\end{restatable}
The proof is deferred to \Cref{apx:proof prelim}.

\section{A Robust and Prior-Agnostic Aggregation Strategy}
\label{sec:aggregation}

\newcommand{\distset}{\mathcal{D}}
\newcommand{\dist}{D}
\newcommand{\predicmarg}{P}
\newcommand{\condiopera}{\circ}
\newcommand{\statediff}{\Delta}
\newcommand{\auxbinary}{\delta}
\newcommand{\posteriorprob}{p}
\newcommand{\generalizedlogodds}{\ell}

In this section, we focus on a widely studied and economically relevant class of information structures (see, e.g., \citealp{ABS-18,KWW-24,GHHKSY-25}) in which experts' signals are independent conditional on the realized state.
We refer to such structures as conditionally independent (CI), and denote by $\inforstructureClassCI$ the set of all CI information structures. 
Formally, for a conditionally independent information structure $\inforstructure \in \Delta(\statespace \times \signalspace)$ with $\signalspace = \signalspace_1 \times \signalspace_2$, we say $\inforstructure$ is CI if for every state $\state\in\statespace$ and every signal profile $(\signal_1, \signal_2)\in\signalspace$, we have
\begin{align*}
    \inforstructure(\signal_1, \signal_2\mid \state)
    = 
    \inforstructure_1(\signal_1 \mid \state) \times 
    \inforstructure_2(\signal_2 \mid \state)~,
\end{align*}
where $\inforstructure_i(\cdot \mid \state)$ is the conditional signal distribution for expert $i$.
When convenient, we write $\inforstructure = \inforstructure_1 \condiopera \inforstructure_2$ to emphasize this conditional-independence factorization.

\subsection{A Log-Odds Aggregation Function}
We propose the following family of aggregation functions:

\begin{definition}[Log-odds aggregation function]
\label{defn:prior-agnostic aggregation scheme}
Fix a parameter $\expertoneweight\in [0, 1]$, 
a {\em log-odds aggregation scheme} is a function $\aggregator_{\expertoneweight} (\cdot, \cdot): [0, 1]^2 \rightarrow [0, 1]$ mapping from the two experts' reports $(\report_1, \report_2)\in[0, 1]^2$ to a forecast as follows:
\begin{align*}
    \aggregator_{\expertoneweight}(\report_1, \report_2) 
    = \frac{1}{\displaystyle 
    1 + \left(\frac{1 - \report_1}{\report_1}\right)^\expertoneweight 
    \cdot \left(\frac{1 - \report_2}{\report_2}\right)^\expertoneweight
    }~.
\end{align*}
Equivalently, let $\logit(p) \triangleq \ln\left(\frac{p}{1-p}\right)$ denote the log-odds function of probability $p\in[0, 1]$, aggregator $\aggregator_{\expertoneweight}$ satisfies that 
\begin{align*}
    \logit(\aggregator_{\expertoneweight}(\report_1, \report_2)) 
    = \expertoneweight \cdot \logit(\report_1) 
    + \expertoneweight \cdot \logit(\report_2)~.
\end{align*}
\end{definition}
To obtain intuitions behind the above aggregation function, we start with characterizing the omniscient Bayesian forecast under the conditionally independent information structure.
Fix a state space $\statespace = \{\state_1, \state_2\}$ with the prior mean $\priormean = \expect[\state\sim\prior]{\state}$:
\begin{align}
\label{eq:bayesbenchmark}
    \bayesaggregator(\report_1, \report_2) 
    & = \prob[\inforstructure]{\outcome = 1\mid \report_1, \report_2} \notag \\
    & = 
    \expect[\outcome\sim \Bernoulli(\state_1)]{\outcome} \cdot \prob[\inforstructure]{\state = \state_1\mid \report_1, \report_2} 
    + 
    \expect[\outcome\sim \Bernoulli(\state_2)]{\outcome} 
    \cdot \prob[\inforstructure]{\state = \state_2\mid \report_1, \report_2} \notag \\
    & = \state_1 \cdot \left(1 - 
    \frac{1}{1 + \frac{\priormean - \state_1}{\state_2 - \priormean} \cdot \frac{\state_2 - \report_1}{\report_1 - \state_1} \cdot \frac{\state_2 - \report_2}{\report_2 - \state_1}}
    \right) 
    + 
    \state_2 \cdot \frac{1}{1 + \frac{\priormean - \state_1}{\state_2 - \priormean} \cdot \frac{\state_2 - \report_1}{\report_1 - \state_1} \cdot \frac{\state_2 - \report_2}{\report_2 - \state_1}}~.
\end{align}
As we can see, this expression highlights two key structural features. First, the joint posterior depends on the two reports only through the product of two generalized odds ratios 
\begin{align*}
    \frac{\state_2 - \report_i}{\report_i - \state_1}~,
\end{align*}
which quantifies how strongly expert $i$'s report tilts beliefs toward the low state relative to the high state. Second, under conditional independence, the experts' evidence combines multiplicatively in odds space: the Bayesian update multiplies their odds contributions and then applies a logistic-type transformation (i.e., the outer $\frac{1}{1+\cdot}$ transformation). 

\xhdr{An additive log-odds representation of Bayesian forecast}
Recall that for any probability $p \in (0, 1)$, the log-odds (or logit) is given by $\logit(p) \triangleq \ln\left(\frac{p}{1-p}\right)$.
Then the omniscient posterior probability of the high state satisfies
\begin{align*}
\logit\left(\prob{\state=\state_2\mid \report_1,\report_2}\right)=\generalizedlogodds_{\state_1,\state_2}(\report_1)+\generalizedlogodds_{\state_1,\state_2}(\report_2)-\ln\frac{\priormean - \state_1}{\state_2 - \priormean}
\end{align*}
where $\generalizedlogodds_{\state_1,\state_2}(\report)
\triangleq 
\ln\frac{\report - \state_1}{\state_2 - \report}$ is the generalized log-odds induced by the endpoints $\state_1 and \state_2$.
Thus, Bayesian aggregation under conditional independence amounts to linear pooling in log-odds: 
the aggregator adds the experts' log-odds contributions (plus the prior log-odds), and then maps back via the logistic function. 

Our aggregation function in \Cref{defn:prior-agnostic aggregation scheme} implements the same linear-in-log-odds principle on the canonical $[0,1]$ scale by taking $\state_1 \gets 0, \state_2 \gets 1$.
In this sense, $\aggregator_\expertoneweight$ can be viewed as a symmetric ``tempered Bayes'' rule: it still adds evidence on the log-odds scale, but uses a parameter $\expertoneweight$ to control the strength of aggregation.

\xhdr{Why $\aggregator_\expertoneweight$ may work for unknown $\statespace$}
When the state space $\statespace=\{\state_1,\state_2\}$ is unknown, an aggregator cannot form the correct generalized odds $\frac{\report-\state_1}{\state_2-\report}$ nor the correct prior term $\ln\frac{\priormean}{1-\priormean}$. 
In other words, the exact Bayesian forecast is unavailable because the aggregator cannot normalize a report $\report$ relative to the true endpoints (i.e., $\state_1, \state_2$). 
The key observation from the Bayesian formula above is that what matters is not the endpoints per se, but a ``monotone transformation'' that turns reports into comparable units of evidence and then adds them.
The log-odds aggregator provides a robust surrogate for this transformation: 
it uses the universal mapping $\report \mapsto \ln\frac{\report}{1-\report}$, which preserves the essential ordering 
\begin{center}
    \emph{``higher reports $\Rightarrow$ stronger evidence for the high state,''}
\end{center}
and combines two experts by adding their evidence due to conditional independence.
Moreover, the parameter $\expertoneweight\in[0,1]$ plays a central robustness role. 
While the canonical scale preserves the ordering of evidence, simply summing logits (i.e., $\expertoneweight=1$) fails to subtract the unknown prior log-odds, effectively double-counting the prior information.\footnote{For binary states $\{\state_1,\state_2\}$, expert $i$'s report $\report_i=\expect{\state\mid \signal_i}$ decomposes as $\generalizedlogodds_{\state_1,\state_2}(\report_i) = \generalizedlogodds_{\state_1,\state_2}(\priormean) + \ln(\inforstructure_i(\signal_i\mid \state_2)/\inforstructure_i(\signal_i\mid \state_1))$, where $\generalizedlogodds_{\state_1,\state_2}(z)\triangleq\ln\frac{z-\state_1}{\state_2-z}$.
Thus, the prior log-odds $\generalizedlogodds_{\state_1,\state_2}(\priormean)$ is embedded in each report. Under conditional independence, the joint posterior log-odds satisfies
\begin{align*}
    \generalizedlogodds_{\state_1,\state_2}(\bayesaggregator) = \generalizedlogodds_{\state_1,\state_2}(\priormean) + \sum\nolimits_{i} \ln(\inforstructure_i(\signal_i\mid \state_2)/\inforstructure_i(\signal_i\mid \state_1)) = \generalizedlogodds_{\state_1,\state_2}(\report_1) + \generalizedlogodds_{\state_1,\state_2}(\report_2) - \generalizedlogodds_{\state_1,\state_2}(\priormean)~.
\end{align*}
Summing log-odds \emph{without} subtracting $\generalizedlogodds_{\state_1,\state_2}(\priormean)$ thus double-counts the prior.
Numerically, if $(\state_1,\state_2)=(0,1)$, $\priormean=0.8$ (odds $4$), and $\report_1=\report_2=\sfrac{8}{9}$ (odds $8$), the naive sum yields odds $64$ ($p\approx 0.985$), whereas the correct Bayesian update yields odds $16$ ($p\approx 0.941$).}
Additionally, full summation implicitly assumes perfect conditional independence, which may not hold in practice.
Hence $\expertoneweight$ can be interpreted as an ``influence weight'' or ``trust weight'' on each expert: 
higher $\expertoneweight$ implies higher confidence in the independence and informativeness of the signals relative to the prior; lower $\expertoneweight$ is more conservative, correcting for the double-counting of the prior and potential correlation between experts. 
Importantly, this rule remains implementable without knowing the primitives: it requires only the reported forecasts $(\report_1,\report_2)$ and a single parameter $\expertoneweight$, and still mirrors the Bayesian logic that independent evidence should be combined by multiplying odds.

\subsection{Establishing the Regret Bounds}
In this section, we first establish that there exists a log-odds aggregator that achieves a regret upper bound of \wtedit{$0.025512$}, and we then establish a regret lower bound of $0.02344$ for our prior-agnostic robust forecasting problem.

\xhdr{Upper bound analysis}
We first present the regret upper bound of our proposed aggregation function. 
\begin{theorem}
\label{thm:regret ub unknown}
With $\expertoneweight = \wtedit{0.585}$, the aggregation function $\aggregator_{\expertoneweight}$ guarantees worst-case regret at most $0.025512$; that is,
\begin{align*}
    \Reg[\inforstructureClassCI]{\aggregator_{0.585}}
    \approx 0.025512 < 0.0256~.
\end{align*}
\end{theorem}

We prove \Cref{thm:regret ub unknown} in two steps:
Step 1 shows that, without loss of generality, it suffices to consider environments in which each expert's information structure generates at most two possible signals. 
Step 2 leverages this reduction to reformulate the worst-case regret of $\aggregator_{\expertoneweight}$ as a finite-dimensional optimization problem involving only \cpedit{seven} parameters. 
Optimizing over these parameters, together with $\expertoneweight$, yields a computable regret upper bound, which we solve numerically to obtain the stated guarantee.

\begin{restatable}[Reduction to binary signals]{lemma}{lemreductionbinary}
\label{lem:reduction_binary}
Fix any state space $\statespace = \{\state_1, \state_2\}$, given any aggregation function $\aggregator$, we have
\begin{align*}
    \sup\nolimits_{\inforstructure\in\inforstructureClassCI} \avgloss{\aggregator\mid \statespace, \inforstructure}
    \;\le\;
    \sup\nolimits_{\inforstructure\in\inforstructureClassCI: |\supp(\inforstructure_i)| \le 2, \forall i\in [2]} \avgloss{\aggregator\mid \statespace, \inforstructure}~.
\end{align*}
That is, the worst case regret can be achieved for the information structures that at most two signals are generated for each expert.
\end{restatable}

\begin{proof}
Fix any state space $\statespace = \{\state_1, \state_2\}$.
Define $\statediff \triangleq \state_2 - \state_1$.

If $\statediff=0$, then the prior distribution is degenerate and we have $\bayesaggregator\equiv \state_1$, so the lemma statement is trivial. 
Hence we may assume $\statediff>0$.

Define an indicator variable $\auxbinary\in\{0,1\}$ such that $\state = \state_1+\statediff\cdot \auxbinary$.
For each expert, consider its posterior probability
\begin{align*}
    \posteriorprob_i(\signal_i)\triangleq 
    \prob[\inforstructure]{\auxbinary=1 \mid \signal_i}\in[0,1]~.
\end{align*}
Then by definition, we have
\begin{equation}
\label{eq:xi-affine}
    \report_i(\signal_i)
    =\expect{\state\mid \signal_i}=\state_1+\statediff\,\expect{\auxbinary\mid \signal_i}=\state_1+\statediff\,\posteriorprob_i(\signal_i).
\end{equation}
Similarly, with the joint posterior probability $\posteriorprob_{12}(\signal_1,\signal_2)\triangleq\prob[\inforstructure]{\auxbinary=1\mid \signal_1,\signal_2}$, 
we have
\begin{equation}
\label{eq:bayes-affine}
    \bayesaggregator(\report_1(\signal_1),\report_2(\signal_2))
    =\expect{\state\mid \signal_1,\signal_2}
    =\state_1+\statediff\,\posteriorprob_{12}(\signal_1,\signal_2)~.
\end{equation}

We now construct a new information structure $\inforstructure\primed$ with the state space $\statespace\primed \triangleq \{0, 1\}$ where its prior distribution $\prior\primed$ satisfies that 
$\prob[\state\sim\prior\primed]{\state = 0} = \prob[\state\sim \prior]{\state = \state_1}$, and 
\begin{align}
    \inforstructure\primed(0, \signal) \gets 
    \inforstructure(\state, \signal), \quad
    \text{if } \state = \state_1, \forall \signal\in\signalspace~; 
    \qquad
    \inforstructure\primed(1, \signal) \gets 
    \inforstructure(\state, \signal), \quad
    \text{if } \state = \state_2, \forall \signal\in\signalspace~.
\end{align}
Then given $(\statespace\primed, \inforstructure\primed)$, we next define the aggregator $\aggregator\primed:[0,1]^2\to\reals$ as follows:
\begin{equation}
\label{eq:f-tilde}
    \aggregator\primed(\posteriorprob_1,\posteriorprob_2)
    \triangleq\frac{\aggregator(\state_1+\statediff \posteriorprob_1,\ \state_1+\statediff \posteriorprob_2)-\state_1}{\statediff}~.
\end{equation}
Then combining \eqref{eq:xi-affine}--\eqref{eq:f-tilde},
\begin{align*}
    \aggregator(\report_1(\signal_1),\report_2(\signal_2))-\bayesaggregator(\report_1(\signal_1),\report_2(\signal_2))
    & = 
    \statediff\cdot \aggregator\primed(\posteriorprob_1(\signal_1), \posteriorprob_2(\signal_2)) + \state_1 - 
    \left(\state_1+\statediff\,\posteriorprob_{12}(\signal_1,\signal_2)\right)
    \\
    & =
    \statediff\left(\aggregator\primed(\posteriorprob_1(\signal_1),\posteriorprob_2(\signal_2))-\posteriorprob_{12}(\signal_1,\signal_2)\right)~,
\end{align*}
and thus we have
\begin{equation*}
\label{eq:loss-scale}
    \avgloss{\aggregator\mid \statespace, \inforstructure}
    =\statediff^2
    \cdot 
    \expect[\inforstructure\primed]{\left(\aggregator\primed(\posteriorprob_1(\signal_1),\posteriorprob_2(\signal_2))-\posteriorprob_{12}(\signal_1,\signal_2)\right)^2} 
    =
    \statediff^2\cdot \avgloss{\aggregator\primed\mid \{0, 1\}, \inforstructure\primed}~.
\end{equation*}
The conditional independence $\signal_1\perp \signal_2\mid \state$ is equivalent to $\signal_1\perp \signal_2\mid \auxbinary$ under the bijection $\state=\state_1+\statediff\auxbinary$, and the signal-space is unchanged. 
Let $\inforstructure\primed = \inforstructure\primed_1 \condiopera \inforstructure\primed_2$.
Thus, it suffices to prove the corresponding statement for the normalized $\{0,1\}$ state space setting:
\begin{equation}
\label{eq:goal-01}
    \sup\nolimits_{\inforstructure\primed\in\inforstructureClassCI} \avgloss{\aggregator\mid \{0, 1\}, \inforstructure\primed}
    \;\le\;
    \sup\nolimits_{\inforstructure\primed\in\inforstructureClassCI: |\supp(\inforstructure_i\primed)| \le 2, \forall i\in [2]} \avgloss{\aggregator\mid \{0, 1\}, \inforstructure\primed}~.
\end{equation}
To prove Eqn.~\eqref{eq:goal-01}, we follow the analysis used in \citet{ABS-18}.
Let $P^i$ be the marginal distribution of $\posteriorprob_i(\signal_i)$ on $[0,1]$:
\begin{align*}
    P^i(A)\triangleq\prob[\inforstructure_i\primed]{\posteriorprob_i\in A}, \quad 
    \forall \text{ Borel set } A\subseteq[0,1]~.
\end{align*}
Then Bayes plausibility implies that
\begin{equation*}
    \int_0^1 \posteriorprob \,P^i(\dd p)=\expect{\posteriorprob_i(\signal_i)}=\prob{\auxbinary=1}=\prior_2~.
\end{equation*}
Then we can reformulate the expected relative loss $\avgloss{\aggregator\primed\mid \{0, 1\}, \inforstructure\primed}$ as follows:
\begin{equation*}
    B_{\prior_2}(\aggregator\primed\mid P^1,P^2)
    \triangleq
    \avgloss{\aggregator\primed\mid \{0, 1\}, \inforstructure\primed}
    =
    \iint_{[0,1]^2}\Phi_{\prior_2}(\posteriorprob_1,\posteriorprob_2)\,P^1(\dd \posteriorprob_1)\,P^2(\dd \posteriorprob_2),
\end{equation*}
where
\begin{equation*}
    \Phi_{\prior_2}(\posteriorprob_1,\posteriorprob_2)
    \triangleq
    \left(\frac{\posteriorprob_1\posteriorprob_2}{\prior_2}+\frac{(1-\posteriorprob_1)(1-\posteriorprob_2)}{1-\prior_2}\right)\left(\aggregator\primed(\posteriorprob_1,\posteriorprob_2)-\bayesaggregatornew(\posteriorprob_1,\posteriorprob_2)\right)^2~.
\end{equation*}
As we can see, $(P^1,P^2)\mapsto B_{\prior_2}(\aggregator\primed\mid P^1,P^2)$ is a bilinear function, i.e., it is affine in each marginal $P^i$.
Now fix $\prior_2\in(0,1)$ and define the convex set of mean-$\prior_2$ probability measures:
\begin{align*}
    K(\prior_2)\triangleq\left\{P\ \text{probability measures on }[0,1] : \int_0^1 p\,P(\dd p)=\prior_2\right\}~.
\end{align*}
Let $K_2(\prior_2)\subseteq K(\prior_2)$ be the subset consisting of measures supported on at most two points.
By the fact that every extreme point of the set $K(\prior_2)$ must be supported on at most two points, we know 
\begin{align*}
    \sup\nolimits_{P^1,P^2\in K(\prior_2)} B_{\prior_2}(\aggregator\primed\mid P^1,P^2)
    =
    \sup\nolimits_{P^1,P^2\in K_2(\prior_2)} B_{\prior_2}(\aggregator\primed\mid P^1,P^2)~.
\end{align*}
We thus finish the proof of Eqn.~\eqref{eq:goal-01}.
\end{proof}

\begin{proof}[Proof of \Cref{thm:regret ub unknown}]
Fix any state space $\statespace = \{\state_1, \state_2\}$ with the prior distribution $\prior$.
With \Cref{lem:reduction_binary}, it is without loss to consider an information structure where each expert receives two signals, namely, $\signalspace_1 = \signalspace_2 = \{0, 1\}$.
Given any aggregation rule $\aggregator:[0,1]^2\to[0, 1]$, the expected loss under the above binary-signal conditionally independent structure $\inforstructure = \inforstructure_1 \condiopera\inforstructure_2$ is
\begin{align*}
    \avgloss{\aggregator\mid \statespace, \inforstructure}
    &=
    \sum\nolimits_{(\signal_1, \signal_2)\in \{0,1\}^2}
    \prob[\inforstructure]{\signal_1, \signal_2}\cdot \left(\aggregator\left(\report_1(\signal_1),\report_2(\signal_2)\right)-\bayesaggregator(\report_1(\signal_1),\report_2(\signal_2))\right)^2~.
\end{align*}
where
$
    \prob[\inforstructure]{\signal_1, \signal_2}
    =
    \prior_1 \cdot \inforstructure_1(\signal_1 \mid \state_1) \cdot 
    \inforstructure_2(\signal_2 \mid \state_1)
    + 
    \prior_2 \cdot \inforstructure_1(\signal_1 \mid \state_2) \cdot 
    \inforstructure_2(\signal_2 \mid \state_2)~.
$

Now define the variables $
a_1 \gets \inforstructure_1(\signal_1 \mid \state_1), 
b_1 \gets \inforstructure_2(\signal_1 \mid \state_1), 
a_2 \gets \inforstructure_1(\signal_1 \mid \state_2), 
b_2 \gets \inforstructure_2(\signal_1 \mid \state_2)$.
Then we can see that we can express $\avgloss{\aggregator\mid \statespace, \inforstructure}$ using only these $4$ variables. 
Thus, fixing any aggregation function $\aggregator_\expertoneweight$ for some $\expertoneweight\in[0, 1]$,
we can reformulate the program $
\Reg[\inforstructureClassCI]{\aggregator_{\expertoneweight}}$ of identifying its worst-case regret as the following optimization  that involves $7$ variables:  
$$
    \Reg[\inforstructureClassCI]{\aggregator_{\expertoneweight}}
    = 
    \sup\nolimits_{\statespace\in [0, 1]^2: \state_1 \le \state_2}
    \sup\nolimits_{\prior_2\in[0, 1]; a_1, b_1, a_2, b_2\in[0, 1]^4}
    \avgloss{\aggregator_{\expertoneweight}\mid \statespace, \inforstructure}~.
$$
After discretizing the space $[0, 1]$ for the parameter $\expertoneweight$, 
we solve the above seven-variable optimization problem at each grid point. We then select the value of $\expertoneweight$ that minimizes the resulting worst-case regret $\Reg[\inforstructureClassCI]{\aggregator_{\expertoneweight}}$. \cpedit{Numerical calculations reveal that the worst-case regret is minimized at $\expertoneweight \approx 0.585$. Under this optimal aggregator, the global maximum of the regret function is attained at the configuration where $\state_1 = 0.21097$, $\state_2 = 1.000$, and the prior mean $\priormean = 0.65866$. 
Specifically, the experts' signal structures $a_i, b_i$ correspond to the following reports:
expert 1 provides $\report_1^{\primed} = 0.21097$ and $\report_1^{\doubleprimed} = 0.87002$; 
expert 2 provides $\report_2^{\primed} = 0.21097$ and $\report_2^{\doubleprimed} = 0.87002$. 
The resulting minimum worst-case regret is equal to $0.025512$.}
\end{proof}

\begin{figure}[H]
    \centering
    \begin{tikzpicture}
    \begin{axis}[
        width=0.8\columnwidth,
        height=6.5cm,
        xlabel={Parameter $\expertoneweight$},
        ylabel={Worst-Case Regret},
        grid=major,
        grid style={dashed,gray!30},
        xmin=0, xmax=1.0,
        ymin=0.0, ymax=0.26,
        legend pos=north east,
        legend style={font=\footnotesize, fill=white, fill opacity=0.8, draw opacity=1, text opacity=1},
        tick label style={font=\footnotesize},
        label style={font=\small},
    ]

    \addplot[
        color=blue!70!black,
        thick,
        mark=none,
        smooth
    ]
    coordinates {
        (0.000, 0.25000) (0.050, 0.18757) (0.100, 0.14778)
        (0.150, 0.11723) (0.200, 0.09420) (0.250, 0.07536)
        (0.300, 0.06011) (0.350, 0.04771) (0.400, 0.04049)
        (0.450, 0.03438) (0.500, 0.02950) (0.530, 0.02715)
        (0.550, 0.02608) (0.570, 0.02560) (0.580, 0.02552)
        (0.585, 0.02551) 
        (0.590, 0.02552) (0.600, 0.02559) (0.630, 0.02607)
        (0.650, 0.02656) (0.700, 0.02816) (0.750, 0.03010)
        (0.800, 0.03225) (0.850, 0.03453) (0.900, 0.03690)
        (0.950, 0.03933) (1.000, 0.03958)
    };
    \addlegendentry{Our Aggregator $\aggregator_{\expertoneweight}$}

    \addplot[
        color=red!80!black,
        dashed,
        thick,
        domain=0:1,
        samples=2
    ]
    {0.023379}; 
    \addlegendentry{Lower Bound (0.023379)} 

    \node[coordinate] (opt) at (axis cs:0.585, 0.02551) {};
    \draw[black, ->, shorten >=2pt] (axis cs:0.72, 0.08) node[above, align=center, font=\scriptsize] {Min Regret\\$\approx 0.02551172$\\at $\expertoneweight \approx 0.585$} -- (opt);
    \end{axis}
\end{tikzpicture}
    \caption{The blue curve is the worst-case regret for the aggregation function $\aggregator_\expertoneweight$ for every $\expertoneweight\in\{0, 0.001, 0.002, \ldots, 1\}$,
    which achieves a global minimum of approximately $0.0255$ when $\expertoneweight \approx 0.585$. 
    The dashed red line represents the theoretical lower bound of $0.023379$.
    }
    \label{fig:gamma_sensitivity_unknown}
\end{figure}
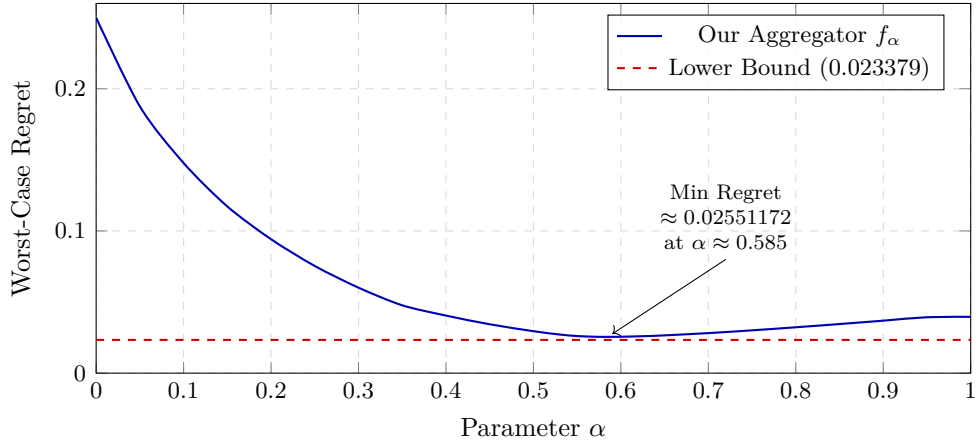

In \citet{ABS-18}, the authors have proposed three aggregation schemes when the state space is known and fixed to be $\statespace = \{0, 1\}$.
In this known-state benchmark,
the omniscient Bayesian aggregator, who knows the information structure and the prior mean $\priormean$, outputs the forecast
\begin{equation*}
    \label{eq:bayes known state}
    \begin{aligned}
    \bayesaggregator\left(\report_1, \report_2\right)
    &=
    \frac{\report_1 \report_2\left(1-\priormean \right)}{\report_1 \report_2\left(1-\priormean\right)+\left(1-\report_1\right)\left(1-\report_2\right)\priormean}\\
    &=
    \frac{1}{1+\left(\frac{1-\report_1}{\report_1}\right)\left(\frac{1-\report_2}{\report_2}\right)\cdot \frac{\priormean}{1-\priormean}}
    \triangleq g\left(\priormean, \report_1, \report_2\right)\,.
    \end{aligned}
\end{equation*}
The three methods proposed in \citet{ABS-18} are summarized as follows:
\begin{itemize}
    \item The Simple Averaging aggregation scheme (i.e., always output the forecast $\frac{\report_1+\report_2}{2}$).
    \item The Average Prior aggregation scheme that outputs a Bayesian forecast by assuming the underlying prior mean is $\frac{\report_1 + \report_2}{2}$, namely, it outputs $g\left(\frac{\report_1 + \report_2}{2}, \report_1, \report_2\right)$.
    \item The Heuristic Prior (SOTA) aggregation scheme that utilizes a heuristic function $ep(\report_1, \report_2)$ to estimate the prior mean, namely, it outputs $g\left(ep(\report_1, \report_2), \report_1, \report_2\right)$.\footnote{The heuristic function is defined as $ep(\report_1, \report_2) = 0.49(\report_1+\report_2) + 0.02 \cdot \mathbb{I}(\report_1+\report_2 > 1)$ in \citet{ABS-18}.}
\end{itemize}
\begin{table}[H]
    \centering
    \renewcommand{\arraystretch}{1.5}
    \begin{tabular}{lcc}
        \toprule
        \textbf{Aggregator} & \textbf{Prior Estimate} ($\hat{\mu}$) & \textbf{Worst-Case Regret} \\
        \midrule
        Simple Averaging & - & $0.0625$ \\
        Average Prior \citep{ABS-18} & $\dfrac{\report_1+\report_2}{2}$ & $0.0311$ \\
        Heuristic Prior \citep{ABS-18} & $ep(\report_1, \report_2)^\dagger$ & $0.0303$ \\
        \citet{KWW-24} ($\lambda=0.8$) & $\dfrac{\report_1+\report_2}{2}$ & $0.0298$ \\ 
        \midrule
        \textbf{Our Aggregator} ($\aggregator_{\expertoneweight}$) & $\mathbf{\expertoneweight \approx 0.585}$ & \textbf{0.025512} \\
        \bottomrule
    \end{tabular}
    \caption{The regret comparison in the unknown state space setting.}
    \label{tab:regret_comparison}
\end{table}
\cpedit{Additionally, we consider the parameterized aggregator family proposed by \citet{KWW-24}. For reported forecasts $(\report_1, \report_2)$, their aggregator is defined as:
\begin{equation}
    \label{eq:kww_aggregator}
    f_{ap}^{\hat{\lambda}}(\report_1, \report_2) = \frac{(1 - \hat{\priormean})^{2\hat{\lambda}-1} \report_1 \report_2}{(1 - \hat{\priormean})^{2\hat{\lambda}-1} \report_1 \report_2 + \hat{\priormean}^{2\hat{\lambda}-1} (1 - \report_1)(1 - \report_2)}~,
\end{equation}
where the prior proxy $\hat{\priormean}$ is set to the arithmetic mean $\frac{\report_1 + \report_2}{2}$. For this scheme, we performed a grid search with a step size of $0.1$ to optimize the parameter $\lambda \in [0, 1]$ to minimize the worst-case regret. In our setting, the optimal value was found at $\lambda=0.8$.

These schemes can also be applied in our unknown-state-space setting by treating the reported forecasts $(\report_1, \report_2)$ as inputs to the same formulas. However, as we demonstrate in Table~\ref{tab:regret_comparison}, even with the optimized $\lambda$, the methods from \citet{ABS-18} and \citet{KWW-24} incur worst-case regrets that are substantially larger than the regret achieved by our proposed aggregation scheme.}

\xhdr{Lower bound analysis}
Recall that our prior-agnostic robust forecasting aggregation problem strictly generalizes the canonical setting studied in prior work, which assumes a fixed and known binary state space $\{0, 1\}$. 
Consequently, any lower bound that holds in that more restrictive model immediately carries over to our setting. In particular, \citet{ABS-18} prove that, under the known $\{0, 1\}$-state-space assumption, every aggregation function incurs worst-case regret at least $\frac{5\sqrt{5}-11}{8} \approx 0.022542$ for this known state space $\{0, 1\}$ setting, directly applies as a valid lower bound in our problem. 
This value therefore remains a valid lower bound for our problem as well. We next show that allowing an unknown state space (i.e., prior-agnosticism) leads to a strictly stronger impossibility result, yielding a strictly larger lower bound on the minimax regret.
\begin{restatable}{theorem}{thmregretlbunknown}
\label{thm:regret lb unknown}
For any aggregation function $\aggregator$, it follows that
\begin{align*}
    \inf\nolimits_{\aggregator} \; 
    \Reg[\inforstructureClassCI]{\aggregator} \ge 
    \frac{31}{1326}
    \approx 0.023379~.
\end{align*}
That is, no aggregation function can guarantee a regret lower
than $0.023379$.
\end{restatable}
\begin{table}[H]
    \centering
    \renewcommand{\arraystretch}{1.5} 
    \begin{tabular}{c c c c c}
        \toprule
        \multirow{2}{*}{\textbf{Structure}} & \multirow{2}{*}{\textbf{State} ($\state$)} & \multirow{2}{*}{\textbf{Prior} ($\prior$)} & \multicolumn{2}{c}{\textbf{Signal Probabilities} $\prob{\signal \mid \state}$} \\
        \cmidrule(lr){4-5}
         & & & $\signal = \signal_L$ & $\signal = \signal_H$ \\
        \midrule
        \multirow{2}{*}{$\inforstructure^A$} & $0$   & $\frac{1}{2}$ & $1$   & $0$ \\
                                             & $\frac{5}{6}$ & $\frac{1}{2}$ & $\frac{1}{4}$ & $\frac{3}{4}$ \\
        \midrule
        \multirow{2}{*}{$\inforstructure^B$} & $\frac{1}{6}$ & $\frac{1}{2}$ & $\frac{3}{4}$ & $\frac{1}{4}$ \\
                                             & $1$   & $\frac{1}{2}$ & $0$   & $1$ \\
        \bottomrule
    \end{tabular}
    \caption{The two symmetric adversarial information structures. $\inforstructure^A$ yields noisy low signals and revealing high signals; $\inforstructure^B$ yields revealing low signals and noisy high signals.}
    \label{tab:adversarial_structures}
\end{table}
To establish the regret lower bound, 
we construct an adversarial distribution over two conditionally independent information structures such that even an aggregator who knows the adversary's mixing rule cannot guarantee regret below $0.023379$.

Specifically, the adversary draws an expert's information structure uniformly at random, i.e., choosing $\inforstructure^A$ or $\inforstructure^B$ with probability $\sfrac{1}{2}$ each.
In both structures, the signal space is binary, $\signalspace = \{\signal_L, \signal_H\}$, and the induced posterior means are $\report_L = \sfrac{1}{6}$ and $\report_H = \sfrac{5}{6}$, respectively (see the detailed construction in Table \ref{tab:adversarial_structures}).
These two information structure differs in their state space and the way that how they generate signals. 
Under $\inforstructure^A$, the state space is $\statespace^A = \{0, \sfrac{5}{6}\}$ with the uniform prior, whereas under $\inforstructure^B$, the state space is $\statespace^B = \{\sfrac{1}{6}, 1\}$ with the uniform prior. 
After the information structure $\inforstructure^\instance$ for some $\instance\in\{A, B\}$ is realized, both experts receive independent private signals generated according to that realized structure $\inforstructure^\instance$, namely, the joint information structure is $\inforstructure = \inforstructure^\instance \condiopera \inforstructure^\instance$. 

\begin{proof}[Proof of \Cref{thm:regret lb unknown}]
Fix any aggregation rule $\aggregator$. 
We construct an adversary's mixed strategy over two conditionally independent information structures $\inforstructure^A, \inforstructure^B$. The adversary draws $\inforstructure^A$ and $\inforstructure^B$ uniformly. 
By Yao's principle, we have that
\begin{align*}
    \Reg[\inforstructureClassCI]{\aggregator}
    =\sup\nolimits_{\statespace, \inforstructure \in\inforstructureClassCI} \avgloss{\aggregator\mid \statespace, \inforstructure}
    \ \ge\ 
    \expect[(\statespace^\instance, \inforstructure^\instance)]{\avgloss{\aggregator\mid \statespace^\instance, \inforstructure^\instance\condiopera \inforstructure^\instance}}~.
\end{align*}
Therefore, we have 
\begin{align*}
    \inf\nolimits_{\aggregator}\Reg[\inforstructureClassCI]{\aggregator}\ \ge\ \inf\nolimits_{\aggregator}\ \expect[(\statespace^\instance, \inforstructure^\instance)]{\avgloss{\aggregator\mid \statespace^\instance, \inforstructure^\instance\condiopera \inforstructure^\instance}}~.
\end{align*}
It remains to lower bound the right-hand side.

In both structures, the experts receive binary signals $\signal \in \{\signal_L, \signal_H\}$. 
Let $\report \in\{\report_L, \report_H\}^2$ be the realized report profile, where $\report_L, \report_H$, by construction, satisfy that
\begin{align*}
    \report_L 
    \triangleq  \sum\nolimits_{\state\in\statespace^\instance} \state\cdot \prob[\inforstructure^\instance]{\state \mid \signal_L} = \frac{1}{6}~;\quad 
    \report_H 
    \triangleq  \sum\nolimits_{\state\in\statespace^\instance} \state\cdot \prob[\inforstructure^\instance]{\state \mid \signal_H} = \frac{5}{6}~.
\end{align*}
Given the expert's information structure $\inforstructure^\instance$ for $\instance\in\{A, B\}$, we use $\bayesaggregator_{\inforstructure^\instance \condiopera \inforstructure^\instance }(\report)$ to denote the omniscient aggregator's forecast given the received experts' reports $\report \in \{\report_L, \report_H\}^2$ under the joint information structure $\inforstructure^\instance \condiopera \inforstructure^\instance$. 
By construction, we have that
\begin{align*}
    \bayesaggregator_{\inforstructure^A \condiopera \inforstructure^A}(\report) & =
    \begin{cases}
     \frac{5}{102} & \text{if } \report = (\report_L, \report_L) \\
     \frac{5}{6} & \text{otherwise~;}
    \end{cases}
    \quad
    \bayesaggregator_{\inforstructure^B \condiopera \inforstructure^B}(\report) =
    \begin{cases}
    \frac{97}{102} & \text{if } \report = (\report_H, \report_H) \\
    \frac{1}{6} & \text{otherwise~.}
    \end{cases}
\end{align*}
We compute the marginal probability for the experts' reports $\report \in \{\report_L, \report_H\}^2$ as follows:
\begin{align*}
    \prob[\inforstructure^A \condiopera \inforstructure^A]{\report_L, \report_L} &= \frac{17}{32}~, \quad
    \prob[\inforstructure^A \condiopera \inforstructure^A]{\report_H, \report_H} = \frac{9}{32}~, \quad
    \prob[\inforstructure^A \condiopera \inforstructure^A]{\report_L, \report_H} =
    \prob[\inforstructure^A \condiopera \inforstructure^A]{\report_H, \report_L} =
    \frac{3}{32}~; \\[0.5em]
    \prob[\inforstructure^B \condiopera \inforstructure^B]{\report_L, \report_L} &= \frac{9}{32}~, \quad
    \prob[\inforstructure^B \condiopera \inforstructure^B]{\report_H, \report_H} = \frac{17}{32}~, \quad
    \prob[\inforstructure^B \condiopera \inforstructure^B]{\report_L, \report_H} =
    \prob[\inforstructure^B \condiopera \inforstructure^B]{\report_H, \report_L} = \frac{3}{32}~.
\end{align*}
We next analyze the aggregator's optimal forecast given that he knows the mixed strategy of the adversary. 
For a report profile $\report$, the optimal forecast $\aggregator^*(\report)$ minimizes the expected squared loss:
\begin{align*}
    \aggregator^*(\report) = \frac{\prob[\inforstructure^A \condiopera \inforstructure^A]{\report} \bayesaggregator_{\inforstructure^A \condiopera \inforstructure^A}(\report) + \prob[\inforstructure^B \condiopera \inforstructure^B]{\report} \bayesaggregator_{\inforstructure^B \condiopera \inforstructure^B}(\report)}{\prob[\inforstructure^A \condiopera \inforstructure^A]{\report} + \prob[\inforstructure^B \condiopera \inforstructure^B]{\report}}~.
\end{align*}
Calculating this for each profile yields:
\begin{align*}
    \aggregator^*(\report_L, \report_L) &= \frac{7}{78}~, \quad
    \aggregator^*(\report_H, \report_H) = \frac{71}{78}~, \quad
    \aggregator^*(\report_L, \report_H) 
    = \aggregator^*(\report_H, \report_L) = \frac{1}{2}~.
\end{align*}
Thus, the aggregator's regret is calculated as:
\begin{align*}
   \expect[(\statespace^\instance, \inforstructure^\instance)]{\avgloss{\aggregator\mid \statespace^\instance, \inforstructure^\instance\condiopera \inforstructure^\instance}}
    & = \sum\nolimits_{\instance\in\{A, B\}} 
    \prob{\inforstructure^\instance}
    \sum\nolimits_{\report\in\{\report_L, \report_H\}^2}
    \left(\aggregator^*(\report) - \bayesaggregator_{\inforstructure^\instance \condiopera \inforstructure^\instance }(\report)\right)^2 \cdot \prob[\inforstructure^\instance\condiopera \inforstructure^\instance]{\report}\\
    & = \frac{31}{1326}~.
\end{align*}
This yields the lower bound $\inf_{\aggregator}\Reg{\aggregator} \ge \frac{31}{1326}$. 
\end{proof}

\subsection{The Superior Performance of Aggregator \texorpdfstring{$\aggregator_{\expertoneweight}$}{} in Known-State Setting}

We also evaluate our proposed aggregation function $\aggregator_\expertoneweight$ in the setting where the state space is known and fixed to be $\statespace = \{0,1\}$. 
In this specific setting, we show that there exists an $\expertoneweight\in[0, 1]$ such that the aggregation function achieves a regret upper bound that almost matches the theoretical lower bound $\frac{5\sqrt{5}-11}{8}\approx 0.0225425$ that has been established in \citet{ABS-18}. 

\begin{proposition}
\label{prop:known_binary_upper_bound}
In the state space setting with $\statespace = \{0, 1\}$, 
the aggregation function with $\aggregator_{\expertoneweight}$ $\expertoneweight = 0.5168$ guarantees a worst-case regret at most $0.022599$;
\begin{align*}
    \Reg[\inforstructureClassCI]{\aggregator_{0.5168}}
    \approx 0.022599 < \cpedit{0.0226}~.
\end{align*}
\end{proposition}
As we can see, the above upper bound ($0.022599$), together with the theoretical lower bound derived from \citet{ABS-18},
demonstrate that the optimality gap of our proposed aggregation function is almost negligible:
\begin{align*}
    \inf\nolimits_{\expertoneweight\in [0, 1]}\; 
    \Reg[\inforstructureClassCI]{\aggregator_{\expertoneweight}}
    - 
    \frac{5\sqrt{5}-11}{8}
    \approx 6 \times 10^{-5}~,
\end{align*}
which implies that in the known binary-state setting, our proposed strategy is near-optimal, extracting essentially all achievable information from the experts' forecasts without requiring knowledge of the underlying signal structure.

To corroborate the tightness of our theoretical bounds, we also conduct a numerical sensitivity analysis of the aggregation strategy with respect to parameter $\expertoneweight$.
\Cref{fig:gamma_sensitivity} plots the worst-case regret when the parameter $\expertoneweight$ varies.
The curve exhibits a global minimum in the neighborhood of $\expertoneweight \approx 0.5168$. \cpedit{Numerical calculations reveal that the worst-case instance for this known $\{0, 1\}$ state space setting is with the prior mean $\priormean = 0.81293$. 
Specifically, the experts' signal structures correspond to the following reports:
expert 1 provides $\report_1^{\primed} = 0.20043$ and $\report_1^{\doubleprimed} = 0.99783$; 
expert 2 provides $\report_2^{\primed} = 0.81293$ and $\report_2^{\doubleprimed} = 0.81293$, which gives a maximum regret $0.022599$.}

\begin{figure}[t]
    \centering
    \begin{tikzpicture}
        \begin{axis}[
            width=0.8\columnwidth,
            height=6.5cm,
            xlabel={Parameter $\expertoneweight$},
            ylabel={Worst-Case Regret},
            grid=major,
            grid style={dashed,gray!30},
            xmin=0, xmax=1.0,
            ymin=0.0, ymax=0.26,
            legend pos=north east,
            legend style={font=\footnotesize, fill=white, fill opacity=0.8, draw opacity=1, text opacity=1},
            tick label style={font=\footnotesize},
            label style={font=\small},
        ]

        \addplot[
            color=blue!70!black,
            thick,
            mark=none,
            smooth
        ]
        coordinates {
            (0.000, 0.25000) (0.050, 0.18757) (0.100, 0.14778)
            (0.150, 0.11775) (0.200, 0.09420) (0.250, 0.07535)
            (0.300, 0.06011) (0.350, 0.04770) (0.400, 0.03757)
            (0.450, 0.02929) (0.480, 0.02508) (0.500, 0.02289)
            (0.510, 0.02263) (0.516, 0.02259) (0.517, 0.02259)
            (0.518, 0.02259) (0.520, 0.02260) (0.530, 0.02266)
            (0.550, 0.02286) (0.600, 0.02370) (0.650, 0.02490)
            (0.700, 0.02638) (0.750, 0.02806) (0.800, 0.02989)
            (0.850, 0.03184) (0.900, 0.03389) (0.950, 0.03600)
            (1.000, 0.03817)
        };
        \addlegendentry{Our Aggregator $\aggregator_{\expertoneweight}$}

        \addplot[
            color=red!80!black,
            dashed,
            thick,
            domain=0:1,
            samples=2
        ]
        {0.02254};
        \addlegendentry{Lower Bound (0.02254)}

        \node[coordinate] (opt) at (axis cs:0.517, 0.02259) {};
        \draw[black, ->, shorten >=2pt] (axis cs:0.65, 0.08) node[above, align=center, font=\scriptsize] {Min Regret\\$\approx 0.0225986$\\at $\expertoneweight \approx 0.517$} -- (opt);

        \end{axis}
    \end{tikzpicture}
    \caption{Sensitivity analysis of the worst-case regret with respect to the parameter $\expertoneweight$. The blue curve depicts the performance of $\aggregator_{\expertoneweight}$. The regret decreases sharply as $\expertoneweight$ approaches $0.517$, achieving a global minimum that converges to the theoretical lower bound (dashed red line).}
    \label{fig:gamma_sensitivity}
\end{figure}
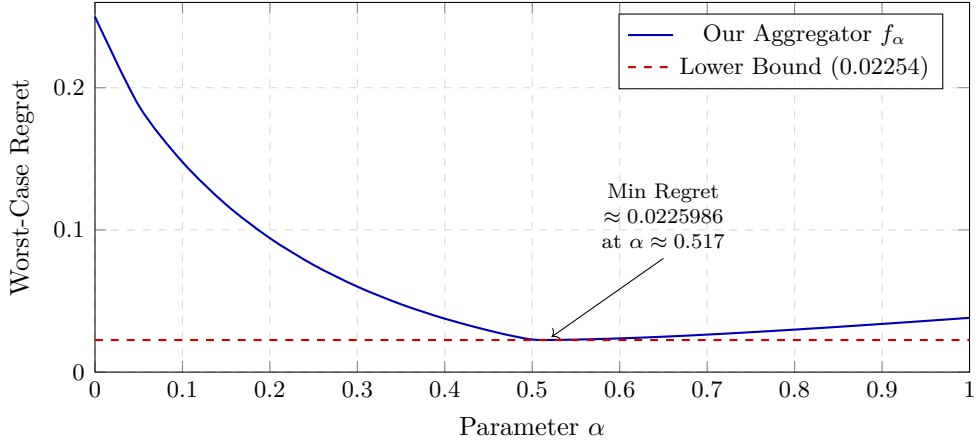

\xhdr{Compare to other aggregation schemes}
As mentioned, \citet{ABS-18} propose three aggregation schemes for the known state space $\statespace = \{0, 1\}$ setting.
However, it is known that none of these aggregation schemes achieve a regret strictly lower than $0.0250$. \cpedit{We also consider the  aggregator family from \citet{KWW-24}. In this known-state benchmark, we optimized the parameter $\lambda \in [0, 1]$ and found that the minimum worst-case regret is achieved at $\lambda = 1$. At this value, the formula in \eqref{eq:kww_aggregator} simplifies to the Bayesian aggregator with a prior $\hat{\mu}$, effectively degenerating to the Average Prior method of \citet{ABS-18}.}

In recent work, \citet{GHHKSY-25} have developed an algorithmic framework for this known-state setting that computes a near-optimal aggregator via efficient approximation procedures (e.g., based on discretized reports or Lipschitz-type regularity restrictions).
They show that their framework produces an aggregator whose regret nearly matches the known lower bound.
However, their aggregator is obtained through a computational pipeline involving discretization, numerical optimization, and interpolation, which leads to an implicit mapping rather than an explicit closed-form aggregation rule.
In contrast, we propose a simple, one-parameter aggregator $\aggregator_\expertoneweight$ (\Cref{defn:prior-agnostic aggregation scheme}) that is directly implementable and easy to use: one simply plugs in the experts' report $(\report_1, \report_2)$ and set $\wtedit{\expertoneweight \gets 0.517}$ to achieve performance that is comparably close to the theoretical lower bound (within $6\times 10^{-5}$), without the overhead of running a separate approximation algorithm or constructing the aggregation rule numerically.

We summarize these comparisons in Table~\ref{tab:regret_comparison}.


\begin{table}[H]
    \centering
    \renewcommand{\arraystretch}{1.5}
    \begin{tabular}{lcc}
        \toprule
        \textbf{Aggregator} & \textbf{Prior Estimate} ($\hat{\mu}$) & \textbf{Worst-Case Regret} \\
        \midrule
        Simple Averaging & - & $0.0625$ \\
        Average Prior \citep{ABS-18} & $\frac{\report_1+\report_2}{2}$ & $0.0260$ \\
        State-of-the-art \citep{ABS-18} & $ep(\report_1, \report_2)^\dagger$ & $0.0250$ \\
        \citet{KWW-24} ($\lambda=1$) & $\frac{\report_1+\report_2}{2}$ & $0.0260$ \\ 
        \citet{GHHKSY-25} & - & $0.0226$ \\
        \midrule
        \textbf{Our Aggregator} ($\aggregator_{\expertoneweight}$) & $\mathbf{\expertoneweight \approx 0.5168}$ & \textbf{0.022599} \\
        \bottomrule
        \multicolumn{3}{p{0.9\linewidth}}{\footnotesize $^{\dagger}$ The heuristic function is $ep(\report_1, \report_2) = 0.49(\report_1+\report_2) + 0.02 \cdot \mathbb{I}(\report_1+\report_2 > 1)$.} \\
    \end{tabular}
    \caption{The regret comparison in the known state space setting with $\statespace = \{0, 1\}$.}
    \label{tab:regret_comparison_known}
\end{table}

\section{Blackwell-Ordered and General Information Structures}
\label{sec:blacwell-ordered and general}

In this section, we analyze the other two classes of information structures:
the Blackwell-ordered setting where the experts' signals can be ranked according to the Blackwell order;
and the general setting where no assumption is imposed on the information structures. 

\subsection{Blackwell-Ordered Information Structures}
In the Blackwell-ordered setting, 
we show that in this setting, the lack of knowledge regarding the state values $\statespace$ and the prior $\prior$ does not penalize the aggregator. 
By leveraging the results from \citet{ABS-18}, we demonstrate that, even without knowledge of the state values, the aggregator achieves the same minimax regret bound as established in the known state space setting.

Let $\inforstructureClassBO$ be the class of Blackwell-ordered information structures. 
Formally, for any $\inforstructure \in \inforstructureClassBO$, we assume without loss of generality that Expert 2 is \emph{Blackwell-more-informative} than Expert 1. 
This implies that Expert 1's signal $\signal_1$ can be viewed as a noisy garbling of Expert 2's signal $\signal_2$ (or equivalently, $\signal_2$ is a sufficient statistic for the pair $(\signal_1, \signal_2)$ with respect to the state $\state$).
\begin{restatable}{corollary}{corblackwell}
\label{pro:blackwell_lowerupperbound_unknown}
For the Blackwell-ordered information structures, 
we have that $\inf_{\aggregator} \Reg[\inforstructureClassBO]{\aggregator} \ge \frac{5\sqrt{5}-11}{8} \approx 0.0225$. 
There exists an aggregation function that achieves 
$\Reg[\inforstructureClassBO]{\aggregator} = \frac{5\sqrt{5}-11}{8}$.
\end{restatable}
To establish the upper bound, we directly adopt the aggregation scheme proposed by \citet{ABS-18} for the Blackwell-ordered setting.
We observe that, although their algorithm was originally derived for a known state space, the construction of the aggregation function relies \emph{solely} on the experts' reported forecasts.
Consequently, we can explicitly reuse this algorithm in our prior-agnostic robust forecasting problem without any modification.

Specifically, we employ the \emph{precision-weighted average} scheme as analyzed in \citet{ABS-18}.
We define the precision of a forecast $\report \in (0, 1)$ as the inverse of its variance, denoted by $\phi(\report) \triangleq \frac{1}{\report(1-\report)}$.
The aggregation function $\aggregator_{\text{pre}}: [0, 1]^2 \rightarrow [0, 1]$ is defined as follows:
\begin{align*}
    \aggregator_{\text{pre}}(\report_1, \report_2) = 
    \begin{cases} 
    \frac{\sqrt{\phi(\report_1)}\report_1 + \sqrt{\phi(\report_2)}\report_2}{\sqrt{\phi(\report_1)} + \sqrt{\phi(\report_2)}} & \text{if } |\report_1 - \report_2| > 0.4 \\[3.5ex]
    \frac{\phi(\report_1)\report_1 + \phi(\report_2)\report_2}{\phi(\report_1) + \phi(\report_2)} & \text{if } |\report_1 - \report_2| \le 0.4
    \end{cases}
\end{align*}
with appropriate boundary handling for extreme forecasts.
Crucially, as this aggregation rule depends \emph{only} on the observed reports $\report$ and is independent of the unknown state values $\statespace$ or the prior $\prior$, it constitutes a valid and robust strategy in our setting. 

\begin{proof}[Proof of \Cref{pro:blackwell_lowerupperbound_unknown}]
We first prove the lower bound, then prove the upper bound.

\xhdr{Lower bound} 
The lower bound follows from an inclusion argument.
Since the adversary is free to select the standard binary state space $\statespace=\{0, 1\}$, the set of admissible environments in our framework (where $\statespace$ is unknown) is a superset of the environments where the state space is known and fixed to be $\{0, 1\}$.
Consequently, the minimax regret in our setting cannot be lower than the minimax regret in the benchmark setting.
Note that in both settings, the aggregator is agnostic to the prior distribution.
Therefore, the worst-case construction from \citet{ABS-18} remains valid, which established a lower bound of $\frac{5\sqrt{5}-11}{8}$, that applies to our setting also.

\xhdr{Upper bound}
The proof relies on establishing an equivalence between our prior-agnostic robust forecasting problem with unknown state space and the optimization problem solved in \citet{ABS-18}.

First, consider the omniscient benchmark. In the class of Blackwell-ordered information structures $\inforstructureClassBO$, one expert is strictly more informed than the other and the optimal strategy for an omniscient aggregator is simply to follow the prediction of the more-informed expert. Since the aggregator does not know the identity of the more-informed expert, we consider the worst-case scenario where the identity is randomized. Without loss of generality for a specific realization, assume Expert 2 is more informed. The signal $\signal_2$ is sufficient for the state $\state$, implying the Bayesian optimal forecast is simply $\bayesaggregator(\report(\signal)) = \report_2$.

Using Lemma~\ref{lem:regret_equivalence} (Regret Simplification), our objective function simplifies to:
\begin{align}
    \Reg[\inforstructureClassBO]{\aggregator_{\text{pre}}} 
    = \expect{\left(\aggregator_{\text{pre}}(\report(\signal)) - \bayesaggregator(\report(\signal))\right)^2}
    = \expect{\left(\aggregator_{\text{pre}}(\report(\signal)) - \report_2 \right)^2}~. \label{eq:reduced_optimization}
\end{align}

We now observe that the optimization problem in Eq.~\eqref{eq:reduced_optimization} is independent of the specific values of the state space $\statespace = \{\state_1, \state_2\}$.
The expectation depends \emph{solely} on the joint distribution of the experts' forecasts $(\report_1, \report_2)$.
Besides, the set of feasible joint distributions of forecasts corresponds exactly to the set of martingales with a fixed prior mean, regardless of the underlying state values that generated these forecasts.

Consequently, minimizing the worst-case regret in our setting reduces to finding an aggregation function $\aggregator$ that minimizes the expected squared distance to the more-informed forecast under the worst-case.
This is mathematically identical to the optimization problem formulated and solved in Theorem 1 of \citet{ABS-18}.
Therefore, their derived bound applies directly to our setting:
\begin{align*}
    \min_{\aggregator} \Reg[\inforstructureClassBO]{\aggregator_\text{pre}} \le \frac{5\sqrt{5}-11}{8}~.
\end{align*}
This establishes that the precision scheme $\aggregator_{\text{pre}}$ remains robust optimal in our unknown state setting, as the guaranteed upper bound matches the earlier derived lower bound.
\end{proof}

\subsection{General Information Structures}

In this section, we analyze the performance of the aggregator under general information structures, where no structural assumptions (such as Blackwell-order or conditionally independent) are imposed on the experts' signals.
We show that in this unrestricted setting, information aggregation is effectively impossible: the optimal strategy is the trivial forecast of $\sfrac{1}{2}$, and the minimax regret is exactly $\sfrac{1}{4}$.

\begin{restatable}{corollary}{corgenerallbubunknow}
\label{pro:general_lbub_unknow}
For the general information structures, 
we have that $\inf_{\aggregator} \Reg[\inforstructureClass]{\aggregator} = 0.25$. 
There exists a trivial aggregation function $\aggregator\equiv 0.5$ that achieves 
$\Reg[\inforstructureClass]{\aggregator} = 0.25$.
\end{restatable} 
\begin{proof}
The proof proceeds by establishing matching upper and lower bounds.

\xhdr{Upper bound}
Consider the trivial aggregation strategy $\aggregator_{\text{triv}}(\report) = 1/2$ for all inputs.
The expected squared loss of this strategy against any outcome $\outcome \in \{0, 1\}$ is:
\begin{align*}
  \Reg[\inforstructureClass]{\aggregator_{\text{triv}}} 
    &= \expect{\left(\frac{1}{2} - \outcome\right)^2} \nonumber \\
    &= \prob{\outcome=1}\left(\frac{1}{2} - 1\right)^2 + \prob{\outcome=0}\left(\frac{1}{2} - 0\right)^2 \nonumber \\
    &= \prob{\outcome=1}\cdot \frac{1}{4} + \prob{\outcome=0}\cdot \frac{1}{4}\\ &= \frac{1}{4}~.
\end{align*}
Therefore, the regret upper bound under general information structures is $0.25$.

\xhdr{Lower bound}
The lower bound follows directly from an inclusion argument.
The adversary is free to choose the standard state space $\statespace=\{0, 1\}$.
Therefore, the minimax regret in our framework is at least as high as in \citet{ABS-18} setting, inheriting the lower bound of 0.25.
\end{proof}

\section{Knowing the Marginal Prediction Distributions}
\label{sec:known-marg}
In our previous discussions, we have assumed that the aggregator observes only the experts‘ reports and has no further information about either the underlying information structure or the state space. 
In this section, we instead consider a setting in which the aggregator additionally knows the marginal distributions of the experts' forecasts (i.e., the unconditional distribution of each expert's posterior mean) \citep{DIL-21,LR-22}, however, the underlying joint information structure and the state space still remain unobservable to the aggregator.

In some environments, this additional knowledge can make aggregation essentially trivial.
For instance, in Blackwell-ordered settings, where one expert is uniformly more informative than the other, this additional knowledge can make the problem essentially trivial: the aggregator can identify the better-informed expert by comparing the informativeness implied by their respective marginal forecast distributions (via the induced Blackwell order) and then simply follow that expert and get a regret of $0$.
Similarly, in the setting of conditionally independent information structures with a fixed and known $\{0, 1\}$ state space, knowing the marginal prediction distribution will also make the problem trivial. 
In this case, each expert's marginal forecast distribution must satisfy Bayes plausibility with respect to the prior distribution, so observing the marginal distribution allows the aggregator to infer the prior mean.
Since, under conditional independence, the omniscient Bayesian benchmark depends on the information structure only through this prior mean (see Eqn.~\eqref{eq:bayesbenchmark})\cpcomment{updated}, the aggregation problem again becomes trivial once the marginals are known.

Given these examples, it is natural to expect that knowing the marginal forecast distributions might generally enable the aggregator to design a more sophisticated rule and achieve strictly smaller worst-case regret.
Perhaps surprisingly, this is not the case. We show first that over the class of general information structures, the aggregator does not gain from this additional information: the minimax (worst-case) regret lower bound is unchanged, matching the benchmark setting in which the aggregator observes only the reports (see \Cref{pro:general_lbub_knowmarg}). 
Moreover, even when we restrict attention to conditionally independent information structures, the minimax regret lower bound remains the same as the lower bound in the setting where the state space is fixed to be $\{0, 1\}$  and known to the aggregator (see \Cref{thm:regret lb unknown w marg}).
We then show that a generalized version of our proposed log-odds aggregation scheme can achieve small regret (see \Cref{pro:con_ub_knowmarg}). 


Below, we first present our results for conditionally independent information structures, and then present our results for the general information structures. 

\subsection{Conditionally Independent Information Structures}
We first present a regret lower bound. 
\begin{theorem}
\label{thm:regret lb unknown w marg}
With knowledge of each expert's marginal forecast distribution, it follows that for any aggregation function $\aggregator$, we have
\begin{align*}
    \inf\nolimits_{\aggregator} \; 
    \Reg[\inforstructureClassCI]{\aggregator} \ge 
    \frac{5\sqrt{5}-11}{8}
    \approx 0.022542~.
\end{align*}
That is, no aggregation function can guarantee a regret lower
than $0.022542$.
\end{theorem}
Recall that in \citet{ABS-18}, they also established a same lower bound under the conditionally independent information structures when the aggregator knows the state space to be $\{0, 1\}$. 
However, their lower bound instance cannot be directly applied to our setting as their constructions are all based on known state space $\{0, 1\}$. Together with the knowledge of the marginal forecast distribution, the aggregator can infer the prior mean and thus achieve $0$ regret. 

To obtain a non-trivial lower bound, one must design information structures with different state spaces. 
In particular, we construct a mixed adversary that randomizes between two single-expert information structures $\inforstructure^A$ and $\inforstructure^B$ in Table~\ref{tab:adversarial_structures_roots}
 (each chosen with probability $1/2$).
Here $\inforstructure^\instance \in \Delta(\statespace^\instance \times \signalspace)$ for $\instance\in \{A, B\}$ specifies one expert's signal distribution over the common signal space $\signalspace = \{\signal_L, \signal_H\}$.
The induced joint information structure is then $\inforstructure^\instance \condiopera\inforstructure^\instance$, 
the two experts' signals are generated (conditionally) independently and identically according to $\inforstructure^\instance$.
The key feature is that, although the two structures have \emph{different state spaces}, they induce the
\emph{same marginal forecast distribution} for each expert. Indeed, in both $\inforstructure^A$ and $\inforstructure^B$, each expert observes a binary signal $\signal\in\{\signal_L,\signal_H\}$ with $\prob{\signal_L}=\prob{\signal_H}=1/2$, and the expert's report (i.e., his
posterior mean) takes the same two values (corresponding to $\signal_L$ and $\signal_H$). 
Consequently, even when the aggregator knows the marginal forecast distributions, this information does not reveal which
information structure has been realized.
That is, from the aggregator's viewpoint, this ``additional information'' is the same under $\inforstructure^A \condiopera\inforstructure^A$ and $\inforstructure^B \condiopera\inforstructure^B$.

The lower bound follows from an indistinguishability argument at the level of joint reports.
Conditional on observing a report pair
$(\report_1,\report_2)$, the aggregator must output a single forecast $\aggregator(\report_1,\report_2)$ that will be used regardless of whether the realized joint structure is $\inforstructure^A \condiopera\inforstructure^A$ or $\inforstructure^B \condiopera\inforstructure^B$.\footnote{Recall that we write $\inforstructure = \inforstructure_1 \circ\inforstructure_2$ to denote a conditionally independent joint information structure $\inforstructure$ obtained by combining the two experts' individual information structures $\inforstructure_1$ and $\inforstructure_2$ (i.e., expert $i$'s marginal information structure is $\inforstructure_i$).} 
However, the omniscient Bayesian benchmark, who knows the realized structure, produces different posterior means under different realization for some report profiles, because the likelihood ratios induced by the two structures differ even though the marginal distributions of reports coincide. 
Therefore, on those profiles the aggregator faces an unavoidable tradeoff: choosing $\aggregator(\report_1,\report_2)$ close to the Bayesian forecast under $\inforstructure^A\condiopera\inforstructure^A$ necessarily incurs a nontrivial squared error under $\inforstructure^B\condiopera\inforstructure^B$, and vice versa, which leads to a non-trivial regret lower bound. 


\begin{table}[H]
    \centering
    \renewcommand{\arraystretch}{1.5} 
    \setlength{\tabcolsep}{10pt}      
    \begin{tabular}{c c c c c}
        \toprule
        \multirow{2}{*}{\textbf{Structure}} & \multirow{2}{*}{\textbf{State} ($\state$)} & \multirow{2}{*}{\textbf{Prior} ($\prior$)} & \multicolumn{2}{c}{\textbf{Signal Probabilities} $\prob{\signal \mid \state}$} \\
        \cmidrule(lr){4-5}
         & & & $\signal =\signal_L$ & $\signal = \signal_H$ \\
        \midrule
        \multirow{2}{*}{$\inforstructure^A$} & $\frac{3-\sqrt{5}}{4}$ & $\frac{\sqrt{5}-1}{2}$ & $\frac{1+\sqrt{5}}{4}$ & $\frac{3-\sqrt{5}}{4}$ \\
                                             & $1$                    & $\frac{3-\sqrt{5}}{2}$ & $0$                    & $1$ \\
        \midrule
        \multirow{2}{*}{$\inforstructure^B$} & $0$                    & $\frac{3-\sqrt{5}}{2}$ & $1$                    & $0$ \\
                                             & $\frac{1+\sqrt{5}}{4}$ & $\frac{\sqrt{5}-1}{2}$ & $\frac{3-\sqrt{5}}{4}$ & $\frac{1+\sqrt{5}}{4}$ \\
        \bottomrule
    \end{tabular}
  \caption{Adversary's mixed strategy of the information structures. In both information structures, each signal is observed with a marginal probability of $\sfrac{1}{2}$.}
    \label{tab:adversarial_structures_roots}
\end{table}

\begin{proof}[Proof of \Cref{thm:regret lb unknown w marg}]
Both the information structures share a common signal space $\signalspace = \{\signal_L, \signal_H\}$.
Information structure $\inforstructure^A \in \Delta(\statespace^A \times \signalspace)$ is defined over the state space $\statespace^A$ with the prior distribution $\prior^A$ defined as follows:
\begin{align*}
    \statespace^A = \left\{\frac{3-\sqrt{5}}{4}, 1\right\}~, \quad \prob[\state\sim \prior^A]{\state = \frac{3-\sqrt{5}}{4}} = \frac{\sqrt{5} - 1}{2}~,
\end{align*}
while information structure $\inforstructure^B \in \Delta(\statespace^B \times \signalspace)$ is defined over the state space $\statespace^B$ with the prior distribution $\prior^B$ defined as follows:
\begin{align*}
    \statespace^A = \left\{0, \frac{1+\sqrt{5}}{4}\right\}~, \quad \prob[\state\sim \prior^B]{\state = \frac{1+\sqrt{5}}{4}} = \frac{\sqrt{5} - 1}{2}~.
\end{align*}
As we can see, both information structure shares the same prior mean $\priormean = 0.5$ although they have different state spaces.

We will then construct each information structure $\inforstructure^\instance$ for $\instance\in\{A, B\}$, such that whenever the expert receives the signal $\signal\in\signalspace$, the expert's posterior mean (i.e., his forecast) over the states satisfies that 
\begin{align*}
    \report_L 
    \triangleq  \sum\nolimits_{\state\in\statespace^\instance} \state\cdot \prob[\inforstructure^\instance]{\state \mid \signal_L} = \frac{3-\sqrt{5}}{4}~;\quad 
    \report_H 
    \triangleq  \sum\nolimits_{\state\in\statespace^\instance} \state\cdot \prob[\inforstructure^\instance]{\state \mid \signal_H} = \frac{1 + \sqrt{5}}{4}~.
\end{align*}
In the meanwhile, the marginal forecast distribution is the same under the information structures $\inforstructure^A$ and $\inforstructure^B$, which is,
\begin{align}
    \label{eq:marg forecast dist}
    \prob[\inforstructure^\instance]{\report_L} = \prob[\inforstructure^\instance]{\report_H} = 0.5, \quad 
    \forall \instance\in\{A, B\}~.
\end{align}
Since each expert receives their signal from the same realized information structure, the expert's marginal forecast distribution is also the same. 


We next analyze the regret of any aggregation rule, even when the aggregator knows the marginal forecast distribution of each expert.
By Yao's principle, we have that
\begin{align*}
    \Reg[\inforstructureClassCI]{\aggregator}
    =\sup\nolimits_{\statespace, \inforstructure \in\inforstructureClassCI} \avgloss{\aggregator\mid \statespace, \inforstructure}
    \ \ge\ 
    \expect[(\statespace^\instance, \inforstructure^\instance)]{\avgloss{\aggregator\mid \statespace^\instance, \inforstructure^\instance\condiopera \inforstructure^\instance}}~.
\end{align*}
Therefore, we have 
\begin{align*}
    \inf\nolimits_{\aggregator}\Reg[\inforstructureClassCI]{\aggregator}\ \ge\ 
    \inf\nolimits_{\aggregator}\ 
    \expect[(\statespace^\instance, \inforstructure^\instance)]{\avgloss{\aggregator\mid \statespace^\instance, \inforstructure^\instance\condiopera \inforstructure^\instance}}~.
\end{align*}
It remains to lower bound the right-hand side.

Given the information structure $\inforstructure^\instance$ for $\instance\in\{A, B\}$, let $\bayesaggregator_{\inforstructure^\instance}(\report)$  denote the omniscient aggregator's forecast given the received report $\report$:
\begin{align*}
    \bayesaggregator_{\inforstructure^A\condiopera \inforstructure^A}(\report) =
    \begin{cases}
    \displaystyle \frac{15-5\sqrt{5}}{4} & \text{if } \report = (\report_H, \report_H) \\
    \displaystyle \frac{3-\sqrt{5}}{4} & \text{otherwise~;}
    \end{cases}
    \quad
    \bayesaggregator_{\inforstructure^B
    \condiopera\inforstructure^B}(\report)  =
    \begin{cases}
    \displaystyle \frac{5\sqrt{5}-11}{4} & \text{if } \report = (\report_L, \report_L) \\
    \displaystyle \frac{1+\sqrt{5}}{4} & \text{otherwise~.}
    \end{cases}
\end{align*}
We compute the marginal probability for the experts reporting $\report$ under each information structure. In this construction, the marginal distributions are symmetric across structures:
\begin{equation}
    \begin{aligned}
    \label{eq:posterior prob marg}
        \prob[\inforstructure^\instance\condiopera \inforstructure^\instance]{\report_L, \report_L} 
        & =
        \prob[\inforstructure^\instance\condiopera\inforstructure^\instance]{\report_H, \report_H} 
        = \frac{1+\sqrt{5}}{8}~; \\
        \prob[\inforstructure^\instance\condiopera\inforstructure^\instance]{\report_L, \report_H} 
        & = \prob[\inforstructure^\instance\condiopera \inforstructure^\instance]{\report_H, \report_L} = \frac{3-\sqrt{5}}{8}~.
    \end{aligned}
\end{equation}
We next analyze the aggregator's optimal forecast given that he knows the mixed strategy of the adversary. 
Let the adversary pick $\instance\in\{A,B\}$ uniformly at random. 
Fix any aggregation function
$\aggregator$ that may depend on the known forecast marginal distribution (i.e., Eqn.~\eqref{eq:marg forecast dist} in our construction). 
Knowing the adversary's mixed strategy, the expected regret of the aggregation function $\aggregator$ is 
\begin{align*}
    & \expect[(\statespace^\instance, \inforstructure^\instance)]{\avgloss{\aggregator\mid \statespace^\instance, \inforstructure^\instance\condiopera \inforstructure^\instance}} \\
    = ~ &  
    \sum\nolimits_{\report\in\{\report_L, \report_H\}^2} 
    \frac{1}{2} \cdot
    \sum\nolimits_{\instance\in\{A, B\}}
    \prob[\inforstructure^\instance\condiopera \inforstructure^\instance]{\report} \left(\aggregator(\report) - \bayesaggregator_{\inforstructure^\instance\condiopera \inforstructure^\instance}(\report) \right)^2 \\
    = ~ &  
    \sum\nolimits_{\report\in\{\report_L, \report_H\}^2}
    \prob[\inforstructure]{\report}\cdot
    \sum\nolimits_{\instance\in\{A, B\}}  
    \frac{1}{2} \cdot
    \left(\aggregator(\report) - \bayesaggregator_{\inforstructure^\instance\condiopera \inforstructure^\instance}(\report) \right)^2
    \tag{By Eqn.~\eqref{eq:posterior prob marg} where
    $\prob[\inforstructure^A\condiopera \inforstructure^A]{\report}
    = 
    \prob[\inforstructure^B\condiopera \inforstructure^B]{\report}$}\\
    = ~ &  
    \sum\nolimits_{\report\in\{\report_L, \report_H\}^2}
    \prob[\inforstructure]{\report}\cdot
    \left(
    \left(\aggregator(\report) - \frac{\bayesaggregator_{\inforstructure^A \condiopera \inforstructure^A}(\report) + \bayesaggregator_{\inforstructure^B \condiopera \inforstructure^B}(\report)}{2}\right)^2
    + 
    \frac{\left(\bayesaggregator_{\inforstructure^A \condiopera \inforstructure^A}(\report) - \bayesaggregator_{\inforstructure^B \condiopera \inforstructure^B}(\report)\right)^2}{4}
    \right) \\
    \ge ~ &   
    \sum\nolimits_{\report\in\{\report_L, \report_H\}^2}
    \prob[\inforstructure]{\report}\cdot
    \frac{\left(\bayesaggregator_{\inforstructure^A \condiopera \inforstructure^A}(\report) - \bayesaggregator_{\inforstructure^B \condiopera \inforstructure^B}(\report)\right)^2}{4} \\
    = ~ &   \frac{1+\sqrt{5}}{8} \cdot \left(
    \left(
    \frac{15 - 5\sqrt{5}}{4}
    - 
    \frac{1 + \sqrt{5}}{4}
    \right)^2 + 
    \left(
    \frac{3 - \sqrt{5}}{4}
    - 
    \frac{5\sqrt{5}-11}{4}
    \right)^2\right) \cdot \frac{1}{4}
    +
    \frac{3 - \sqrt{5}}{8} \cdot 
    2 \cdot  \left(
    \frac{3 - \sqrt{5}}{4}\!\!
    - \!\!
    \frac{1 + \sqrt{5}}{4}
    \right)^2 \cdot \frac{1}{4}\\
    = ~ &   \frac{5\sqrt{5}-11}{8}~,
\end{align*}
where in the inequality we have used the fact that for all $a, b, c\in\reals$, we have 
\begin{align*}
    \frac{1}{2} \left((c-a)^2 + (c-b)^2\right)
    = \left(c - \frac{a+b}{2}\right)^2 + \frac{(a-b)^2}{4}
    \ge \frac{(a-b)^2}{4}~.
\end{align*}
We thus have finished the proof.
\end{proof}

We now show that there exists a generalized log-odds aggregation function that can achieve a regret upper bound of $0.022763$. \wtcomment{update}\cpcomment{updated}
\begin{definition}[Generalized log-odds aggregation function]
\label{defn:generalized prior-agnostic aggregation scheme}
Let $\priormean\in[0, 1]$ be the prior mean.
Fix a parameter $\expertoneweight\in [0, 1]$ and a parameter $\baserate\in[-1, 1]$, 
a {\em generalized log-odds aggregation scheme} is a function $\aggregator_{\expertoneweight, \baserate} (\cdot, \cdot): [0, 1]^2 \rightarrow [0, 1]$ mapping from the two experts' reports $(\report_1, \report_2)\in[0, 1]^2$ to a forecast as follows:
\begin{align*}
    \aggregator_{\expertoneweight, \baserate}(\report_1, \report_2) 
    = \frac{1}{\displaystyle
    1 + \left(\frac{\mu}{1 - \mu}\right)^\baserate \cdot \left(\frac{1 - \report_1}{\report_1}\right)^\expertoneweight
    \cdot \left(\frac{1 - \report_2}{\report_2}\right)^\expertoneweight
    }~.
\end{align*}
Equivalently, aggregator $\aggregator_{\expertoneweight, \baserate}$ satisfies that 
\begin{align*}
    \logit(\aggregator_{\expertoneweight, \baserate}(\report_1, \report_2)) =\cpedit{\expertoneweight \cdot} \logit(\report_1) + \cpedit{\expertoneweight \cdot}\logit(\report_2) - \cpedit{\baserate \cdot}\logit(\mu)~.
\end{align*}
\end{definition}
It is important to note that our proposed generalized log-odds aggregation function does not fully leverage the knowledge of the expert's marginal prediction distributions. Instead, it only uses the inferred prior mean in the aggregation function. Nevertheless, as we show below, even with this minimal piece of information there exists a choice of parameters $(\expertoneweight, \baserate)$ that guarantees a uniformly low regret.
\begin{restatable}{proposition}{proconubknowmarg}
\label{pro:con_ub_knowmarg}
With $\expertoneweight = 0.656089$ and $\baserate = 0.498268$, the aggregation function $\aggregator_{\expertoneweight, \baserate}$ guarantees worst-case regret at most $0.0228$; that is,
\begin{align*}
    \Reg[\inforstructureClassCI]{\aggregator_{0.656089, 0.498268}} \approx 0.022763 < 0.0228~.
\end{align*}
\end{restatable}
\cpedit{In the known-marginal distribution setting, the aggregator has access to the marginal distribution of each expert's report. Consequently, the prior mean $\priormean$ is no longer an unknown parameter to be estimated but can be directly identified. To evaluate the performance of our proposed scheme in this benchmark, we consider three baseline algorithms where the prior estimate $\hat{\priormean}$ is set to the true prior mean $\priormean$. First, we consider a trivial baseline that \textit{always outputs the prior mean}. Second, we evaluate the \textit{Standard Bayesian} aggregator from \citet{ABS-18} by substituting the true $\priormean$ into their formula. Finally, we include the parameterized aggregator from \citet{KWW-24}, where the parameter $\lambda$ is optimized to minimize the worst-case regret given the known prior mean. We summarize the regret comparisons in Table~\ref{tab:regret_comparison_knownmarg}.}
\begin{table}[H]
    \centering
    \renewcommand{\arraystretch}{1.5}
    \begin{tabular}{lcc}
        \toprule
        \textbf{Aggregator} & \textbf{Prior Estimate} ($\hat{\priormean}$) & \textbf{Worst-Case Regret} \\
        \midrule
        Always Output Prior Mean & $\priormean$ & $0.2500$ \\
        Standard Bayesian \citep{ABS-18} & $\priormean$ & $0.0403$ \\
        \citet{KWW-24} ($\lambda = 0.8$) & $\priormean$ & $0.0389$ \\
        \midrule
        \textbf{Our Aggregator} ($\aggregator_{\expertoneweight,\baserate}$) & $\mathbf{\expertoneweight \approx 0.656089, \baserate \approx 0.498268}$ & \textbf{0.022763} \\
        \bottomrule
    \end{tabular}
    \caption{The regret comparison in the known-marginal setting.} 
    \label{tab:regret_comparison_knownmarg}
\end{table}

\begin{proof}[Proof of \Cref{pro:con_ub_knowmarg}]
    The proof follows a similar structure to that of \Cref{thm:regret ub unknown}.
    Invoking \Cref{lem:reduction_binary}, it suffices to consider information structures where each expert receives at most two signals, i.e., $\signalspace_1 = \signalspace_2 = \{0, 1\}$.
    
    We parameterize the conditional signal probabilities as $a_1 = \inforstructure_1(\signal_1 \mid \state_1)$, $b_1 = \inforstructure_2(\signal_1 \mid \state_1)$, $a_2 = \inforstructure_1(\signal_1 \mid \state_2)$, and $b_2 = \inforstructure_2(\signal_1\mid \state_2)$.
    Accordingly, the problem of identifying the worst-case regret reduces to a finite-dimensional optimization over the state parameters $\state_1, \state_2$, the prior $\prior$, and the signal probabilities:
    \begin{align*}
        \Reg[\inforstructureClassCI]{\aggregator_{\expertoneweight, \baserate}}
        = 
        \sup\nolimits_{\statespace\in [0, 1]^2: \state_1 \le \state_2}
        \sup\nolimits_{\prior_2\in[0, 1]; a_1, b_1, a_2, b_2\in[0, 1]^4}
        \avgloss{\aggregator_{\expertoneweight, \baserate}\mid \statespace, \inforstructure}~.
    \end{align*}
    By solving the optimization problem for the known-marginal setting, we identify the optimal parameters as $\expertoneweight \approx 0.656089$ and $\baserate \approx 0.498268$. 
    Numerical calculations reveal that the global maximum of the regret function is attained at the configuration where $\state_1 = 0$, $\state_2 = 0.79570$, and the prior mean $\priormean = 0.49908$. 
    Specifically, the experts' signal structures correspond to the following reports:
    expert 1 provides $\report_1^{\primed} = 0.18165$ and $\report_1^{\doubleprimed} = 0.79570$; 
    expert 2 provides $\report_2^{\primed} = 0.18165$ and $\report_2^{\doubleprimed} = 0.79570$. 
    The resulting maximum regret is equal to $0.022763$.
\end{proof}

\subsection{General Information Structures}
We next discuss the regret for the general information structures. 
\begin{restatable}{proposition}{progenerallbubknowmarg}
\label{pro:general_lbub_knowmarg}
For the general information structures, 
we have that $\inf_{\aggregator} \Reg[\inforstructureClass]{\aggregator} = 0.25$. 
There exists a trivial aggregation function $\aggregator\equiv 0.5$ that achieves 
$\Reg[\inforstructureClass]{\aggregator} = 0.25$.
\end{restatable}
\begin{proof}
The upper bound is straightforward as one can use the trivial aggregation function $\aggregator_{\text{triv}}(\report) = 1/2$, which can guarantee a regret of $0.25$.

We next show the lower bound. We use the same hard information structure  proposed in \citet{ABS-18}. 
This instance demonstrates that even if the marginal distribution of each expert's forecast is known to the aggregator, the worst-case regret still remains at least $\sfrac{1}{4}$.

Consider a binary state space $\statespace = \{0, 1\}$ with a uniform prior $\prior(0)=\prior(1)=\sfrac{1}{2}$. 
Experts receive binary signals from $\signalspace = \{\signal_L, \signal_H\}$. 
We construct this information structure as Table~\ref{tab:xor_joint}.

\begin{table}[H]
    \centering
    \renewcommand{\arraystretch}{1.5}
    \setlength{\tabcolsep}{10pt} 
    \begin{tabular}{c c c c c c}
        \toprule
        \multirow{2}{*}{\textbf{State} ($\state$)} & \multirow{2}{*}{\textbf{Prior} ($\prior$)} & \multicolumn{4}{c}{\textbf{Joint Signal Probabilities} $\prob{\signal_1, \signal_2 \mid \state}$} \\
        \cmidrule(lr){3-6}
         & & $(\signal_L, \signal_L)$ & $(\signal_H, \signal_H)$ & $(\signal_L, \signal_H)$ & $(\signal_H, \signal_L)$ \\
        \midrule
        $0$ & $\frac{1}{2}$ & $\frac{1}{2}$ & $\frac{1}{2}$ & $0$ & $0$ \\
        $1$ & $\frac{1}{2}$ & $0$ & $0$ & $\frac{1}{2}$ & $\frac{1}{2}$ \\
        \bottomrule
    \end{tabular}
    \caption{The joint information structure.}
    \label{tab:xor_joint}
\end{table}

Under the constructed information structure, each of the four possible signal profiles $(\signal_1, \signal_2) \in \{\signal_L, \signal_H\}^2$ has an equal unconditional probability of $\sfrac{1}{4}$. 
Moreover, it is easy to see that $\report_i(\signal_L) = \report_i(\signal_H) = \sfrac{1}{2}$ all $i\in[2]$.
In other words, each expert's forecast marginal distribution is a point mass on $\sfrac{1}{2}$.
However, the omniscient forecasts can always infer the true state.
Thus, the aggregator has a relative regret at least $\sfrac{1}{4}$.
\end{proof}

\newpage
\bibliographystyle{ACM-Reference-Format}
\bibliography{mybib}

\newpage
\appendix
\section{Missing Proof in \texorpdfstring{\Cref{sec:prelim}}{prelim}}
\label{apx:proof prelim}
\lemregretequivalence*
\begin{proof}
Recall that the omniscient aggregator reports the true posterior mean given the signal profile $\signal$. By the law of iterated expectations and the fact that $\expect{\outcome \mid \state} = \state$ due to $\outcome\sim \Bernoulli(\state)$, we have:
\begin{align}
    \bayesaggregator(\report(\signal)) 
    = \prob{\outcome = 1 \mid \signal} 
    = \expect{\outcome \mid \signal} = \expect{\expect{\outcome \mid \state} \mid \signal} = \expect{\state \mid \signal}~.
    \label{eq:oracle_equivalence_app} 
\end{align}
Thus, the omniscient forecast $\bayesaggregator(\report(\signal))$ is the conditional expectation of both the outcome $\outcome$ and the state $\state$.

Fix any state space, any information structure, and any aggregation function $\aggregator$, by definition, we have
\begin{align*}
    \avgloss{\aggregator \mid \statespace, \inforstructure} 
    = ~& \expect{\left(\aggregator(\report(\signal)) - \outcome\right)^2} - \expect{\left(\bayesaggregator(\report(\signal))- \outcome\right)^2} \\
    = ~& 
     \expect{\left(\aggregator(\report(\signal)) - \bayesaggregator(\report(\signal))\right)^2} + 2\expect{\left(\aggregator(\report(\signal)) - \bayesaggregator(\report(\signal))\right) 
     \cdot \left(\bayesaggregator - \outcome\right)}~.
\end{align*}
We now show that the cross-term equals to $0$. Conditioning on the signals $\signal$:
\begin{align*}
    & \expect{\left(\aggregator(\report(\signal)) - \bayesaggregator(\report(\signal))\right) 
     \cdot \left(\bayesaggregator(\report(\signal)) - \outcome\right)}  \\
    = ~ & 
    \expect{\expect{\left(\aggregator(\report(\signal)) - \bayesaggregator(\report(\signal))\right) 
     \cdot \left(\bayesaggregator(\report(\signal)) - \outcome\right) \mid \signal}} \\
     = ~ & 
    \expect{\left(\aggregator(\report(\signal)) - \bayesaggregator(\report(\signal))\right) 
    \cdot \expect{\left(\bayesaggregator(\report(\signal)) - \outcome\right) \mid \signal}} \\
     = ~ & 
    \expect{\left(\aggregator(\report(\signal)) - \bayesaggregator(\report(\signal))\right) 
    \cdot \left(\bayesaggregator(\report(\signal)) - \expect{\outcome \mid \signal}\right)} \\
    = ~ & 0~. \tag{Due to Eqn.~\eqref{eq:oracle_equivalence_app}.}
\end{align*}
Consequently, we have $\avgloss{\aggregator \mid \statespace, \inforstructure} = \expect{\left(\aggregator(\report(\signal)) - \bayesaggregator(\report(\signal))\right)^2}$. 
\end{proof}

\end{document}